\documentclass[10pt,twocolumn,letterpaper]{article}

\usepackage[cvpr]{cvpr}      

\usepackage{graphicx}
\usepackage{amsmath}
\usepackage{amssymb}
\usepackage{booktabs}

\usepackage[pagebackref,breaklinks,colorlinks]{hyperref}
\usepackage{url}
\usepackage{booktabs}
\usepackage{graphicx}
\usepackage{wrapfig}
\usepackage{caption}
\usepackage{amsthm}
\newbox\statebox

\newcommand{\myState}[1]{%
    \setbox\statebox=\vbox{#1}%
    \edef\thealgruleheight{\dimexpr \the\ht\statebox+1pt\relax}%
    \edef\thealgruledepth{\dimexpr \the\dp\statebox+1pt\relax}%
    \ifdim\thealgruleheight<.75\baselineskip
        \def\thealgruleheight{\dimexpr .75\baselineskip+1pt\relax}%
    \fi
    \ifdim\thealgruledepth<.25\baselineskip
        \def\thealgruledepth{\dimexpr .25\baselineskip+1pt\relax}%
    \fi
    \State #1%
    \def\thealgruleheight{\dimexpr .75\baselineskip+1pt\relax}%
    \def\thealgruledepth{\dimexpr .25\baselineskip+1pt\relax}%
}
\usepackage{multicol}
\errorcontextlines\maxdimen
\usepackage[ruled, vlined]{algorithm2e}

\newtheorem{property}{Property}

\newtheorem{assumption}{Assumption}
\newtheorem{definition}{Definition}
\newtheorem{lemma}{Lemma}

\newtheorem*{remark}{Remark}

\usepackage{amsmath,amsfonts,bm}

\def\eqref#1{equation~\ref{#1}}

\def\1{\bm{1}}

\def\vg{{\bm{g}}}

\def\vp{{\bm{p}}}

\def\vv{{\bm{v}}}
\def\vw{{\bm{w}}}
\def\vx{{\bm{x}}}

\DeclareMathAlphabet{\mathsfit}{\encodingdefault}{\sfdefault}{m}{sl}
\SetMathAlphabet{\mathsfit}{bold}{\encodingdefault}{\sfdefault}{bx}{n}

\DeclareMathOperator*{\argmin}{arg\,min}

\usepackage[capitalize]{cleveref}
\crefname{section}{Sec.}{Secs.}
\Crefname{section}{Section}{Sections}
\Crefname{table}{Table}{Tables}
\crefname{table}{Tab.}{Tabs.}

\newcommand{\Appendix}{\textbf{Appendix}\xspace}

\begin{document}

\title{Robust Federated Learning against both Data Heterogeneity and Poisoning Attack via Aggregation Optimization}


\author{Yueqi Xie\footnotemark[1]\ , Weizhong Zhang\footnotemark[1]\ , Renjie Pi\footnotemark[1]\ , Fangzhao Wu,\\ Qifeng Chen, Xing Xie,\ Sunghun Kim\\
The Hong Kong University of Science and Technology\\
Microsoft Research Asia \\
}

\author{
Yueqi Xie$^1$\footnotemark[1]\
\quad Weizhong Zhang$^1$\footnotemark[1]\
\quad Renjie Pi$^1$\footnotemark[1]\
\quad Fangzhao Wu$^2$
\quad Qifeng Chen$^1$\\
\quad Xing Xie$^2$
\quad Sunghun Kim$^1$ \\
$^1$The Hong Kong University of Science and Technology \\
$^2$Microsoft Research Asia\\
}

\maketitle

\renewcommand{\thefootnote}{\fnsymbol{footnote}}
\footnotetext[1]{Joint first authors}

\begin{abstract}
Non-IID data distribution across clients and poisoning attacks are two main challenges in real-world federated learning (FL) systems. While both of them have attracted great research interest with specific strategies developed, no known solution manages to address them in a unified framework. To universally overcome both challenges, we propose SmartFL, a generic approach that optimizes the server-side aggregation process with a small amount of proxy data collected by the service provider itself via a subspace training technique. Specifically, the aggregation weight of each participating client at each round is optimized using the server-collected proxy data, which is essentially the optimization of the global model in the convex hull spanned by client models. Since at each round, the number of tunable parameters optimized on the server side equals the number of participating clients (thus independent of the model size), we are able to train a global model with massive parameters using only a small amount of proxy data (e.g., around one hundred samples). With optimized aggregation, SmartFL ensures robustness against both heterogeneous and malicious clients, which is desirable in real-world FL where either or both problems may occur. We provide theoretical analyses of the convergence and generalization capacity for SmartFL. Empirically, SmartFL achieves state-of-the-art performance on both FL with non-IID data distribution and FL with malicious clients. The source code will be released.
\end{abstract}

\section{Introduction}
\label{sec:intro}
Data security and privacy have raised increasing interest in machine learning and computer vision research, especially in privacy-sensitive areas such as health care~\cite{rieke2020future,kairouz2021advances}.
Federated Learning (FL) emerges as an effective privacy-preserving machine learning approach to jointly optimize a global model over decentralized data~\cite{konevcny2016federated,yang2019federated}. Typically, generic FL involves multiple rounds of clients' local training followed by server-side aggregation. The \textbf{server-side aggregation} plays an essential role that aggregates the client models into a global model, which is then used to initialize the clients in the next training round.
The standard aggregation strategy Federated Averaging (FedAVG)~\cite{mcmahan2017communication}, which takes the sample number weighted average over clients' weights, is shown to converge to an ideal model as centralized training and works well in IID data distribution without poisoning attacks~\cite{zinkevich2010parallelized,mcmahan2017communication,zhou2017convergence}.

However, in real-world FL, data heterogeneity across clients and the potential presence of malicious clients severely compromise the effectiveness of standard aggregation~\cite{konevcny2016federated,yang2019federated,yin2018byzantine}. 
Various specifically-designed strategies have been proposed to tackle these two problems separately.
To tackle data heterogeneity, prior studies propose regularized local training~\cite{li2020federated,karimireddy2019scaffold}, personalized FL~\cite{t2020personalized,zhang2021parameterized}, handcrafted aggregation rules to reweight the updates based on the statistics of updates or performance on proxy data~\cite{wang2020tackling,yeganeh2020inverse, xiao2021novel}, and further tuning the global model with proxy data~\cite{lin2020ensembleFedDF,chen2020fedbe} in every communication round. 
To improve robustness against attacks, Byzantine-robust aggregations~\cite{yin2018byzantine,blanchard2017machine,guerraoui2018hidden} are proposed to exclude statistical outliers of updates. In many methods, proxy data
is utilized to provide additional clues for enhanced performance~\cite{sageflow, fltrust}. However, with specifically-crafted strategies for one of the challenges, existing solutions cannot handle both challenges in one generic FL framework, which limits their effectiveness in real-world FL where either or both problems may occur.


\textit{Can we optimize the aggregated global model to jointly handle any potential challenge from heterogeneous and malicious clients, instead of heuristically developing specific methods for one specific problem?} Based on the work presented in this paper, our answer to this question is \textit{Yes}.
Following an emerging line of FL research~\cite{xiao2021novel,fltrust,sageflow,nagalapatti2021game}, we consider the practical scenario that the service provider itself can collect a small amount of clean proxy data for the current learning task. 
With the server-collected data, a straightforward data-driven optimization strategy would be further finetuning the global model aggregated with FedAVG in every communication round. We term those approaches \textit{full-space training} since they optimize the global model in the entire parameter space. However, to tune the global model with massive parameters, a large amount of proxy data is required as the carrier of knowledge. Otherwise, severe overfitting may occur, which is verified in the experimental section (see Section \ref{s_exp_indepth}). Unfortunately, 
it is impractical for the service provider to collect lots of on-server proxy data. 
What's more, the full-space training approaches are unlikely to mitigate the negative effects of malicious clients with limited proxy data since including such clients in aggregation often leads to a drastic performance drop (see Section \ref{exp-attack}).
Also, full-space training leads to low aggregation efficiency and long latency because of the large dataset used and huge amounts of parameters to optimize. Finally, it remains unclear whether the full-space training-based FL systems can be theoretically guaranteed to converge to the optimum.

In light of the above-mentioned issues, we propose SmartFL with a generic and elegant
aggregation strategy that optimizes the aggregated global model
via \textit{subspace training}
within the convex hull spanned by the client models' parameters.
To be precise, each time after local training, SmartFL updates the global model to be the optimal convex combination of the received client models' parameters by fitting the on-server proxy data.
This extracted subspace is mainly inspired by two facts. On the one hand, 
prior studies on mode connectivity~\cite{garipov2018loss,draxler2018essentially,kuditipudi2019explaining} show that low-cost solutions found by two networks can be connected by simple (e.g., piece-wise linear) paths with constant error or loss. 
On the other hand, this subspace naturally contains the reweighting-based aggregation methods for heterogeneous FL and attack-robust FL~\cite{yeganeh2020inverse,xiao2021novel, wang2020tackling,sageflow}. These facts suggest that the extracted subspace 
has the potential to contain the desirable global model.
By optimizing the global model within the subspace, the degree of freedom for training is significantly reduced compared with full-space training.
This makes SmartFL enjoy a much lower demand for on-server proxy data, better generalization, higher efficiency,
and effectiveness to alleviate the effect of attackers with their weights optimized to very low values.
With optimized aggregation in every round, SmartFL ensures robustness against potential challenges from both heterogeneous and malicious clients.
We also establish theoretical guarantees on the convergence and generalization of SmartFL.

It is worth mentioning that our setup is practical, which assumes the service provider itself collects a small clean labelled proxy dataset (around a hundred samples, 0.2\% of the dataset by default). 
The required amount of proxy data for SmartFL is among the smallest ones in the existing work leveraging  server-collected labelled proxy data~\cite{xiao2021novel,cheng2021fedgems,fltrust,sageflow}
Also, as shown in Section~\ref{s_exp_indepth}, SmartFL can boost the performance even when the server-collected data is \textbf{highly different from global distribution}, which further verifies the feasibility of SmartFL.
What's more, we also extend to the usage of a small amount of \textbf{unlabelled data} (SmartFL-U) for heterogeneous FL to empower usage for extreme conditions.
Specifically, we optimize the combination coefficients for labelled data with ground truth labels and unlabelled data with pseudo-labels generated by the ensemble of clients.

We conduct extensive experiments on CIFAR-10/100, FMNIST, MNIST, and 20Newsgroups. The results demonstrate that SmartFL can boost the performance of FL with non-IID data distribution and poisoning attacks with very few proxy data samples. For instance, with only 128 samples (0.2\% of the dataset) of server proxy data for CIFAR-10, we can attain a significant performance improvement compared with state-of-the-art methods for heterogeneous FL~\cite{lin2020ensembleFedDF,chen2020fedbe,xiao2021novel,li2020federated, karimireddy2019scaffold}. 
Also, when malicious clients exist, our solution manages to learn small coefficients for malicious clients to defend against the attacks even in the condition of a large portion of attacks and highly-non-IID data distribution, yielding state-of-the-art performance compared with existing attack-robust methods with proxy data~\cite{fltrust,sageflow} and statistical methods~\cite{yin2018byzantine,blanchard2017machine}.
Our contributions can be summarized as follows:
\begin{itemize}
    \item As far as we know, SmartFL is the first FL framework that universally handles two major challenges in FL systems (i.e., non-IID distribution of data and poisoning attacks) in a unified framework.
    \item We propose SmartFL, which effectively optimizes server-side aggregation with a small amount of server-collected proxy data via subspace training.
    \item We provide theoretical analysis for convergence and generalization capacity for SmartFL. Extensive experiments on multiple datasets with non-IID data distribution and poisoning attacks demonstrate the superiority of our method.
\end{itemize}

\section{Related Work}

\label{related_work}
\subsection{Federated Learning with Non-IID Data}
Increasing research efforts are devoted to improving the FL performance with heterogeneous data distribution.  
They can be classified into modifying local training and modifying server-side aggregation.
In this section, we focus on the latter one, which is more closely related to our work. More related works on improving local training~\cite{li2020federated,Acar2021Dyn,li2021model,karimireddy2019scaffold,shin2020xor,oh2020mix2fld,YoonSHY21,zhao2018federated,kulkarni2020survey,t2020personalized,hanzely2020lower,li2021ditto,chen2021bridging} are discussed in \Appendix Sec. 1.

Several prior studies propose to reweight the model updates with some statistical property and hand-crafted rules. 
FedNova~\cite{wang2020tackling} and IDA~\cite{yeganeh2020inverse} propose to normalize the aggregation weights according to the local training steps and distance between local and global updates, respectively.
FedAvgM~\cite{hsu2019measuring} goes beyond the weighted average and 
adopts server-side momentum to improve the aggregation.
Recently, solutions leveraging server unlabelled/labelled data to further tune the aggregated global model in every communication round have drawn much research attention with promising performance.
Specifically, 
FedDF~\cite{lin2020ensembleFedDF} and   FedBE~\cite{chen2020fedbe} leverage ensemble knowledge with average/bayesian ensembled logits of clients' predictions on the server unlabelled data to finetune the global model and
validate that the ground truth labels 
lead to the best finetuning performance if available. FedET~\cite{cho2022heterogeneous} and FedAUX~\cite{Fedaux} include more carefully logit ensembling strategies.
However, these solutions demand a large amount of proxy data to tune the global model with massive parameters, which is not always realistic for FL systems, even for unlabelled data. 
ABAVG~\cite{xiao2021novel} uses accuracy on labelled proxy data to determine the aggregation weight of clients to enable quality-aware aggregation. However, it heuristically assumes the coefficients should be proportional to the proxy data accuracy, which does not fully utilize the ground truth knowledge and does not get a pleasant gain.
\subsection{Federated Learning with Poisoning Attack}
FL is vulnerable to poisoning attacks due to a vast number of uncontrolled clients, some of which may be malicious~\cite{poisoning_attack}.  
\cite{blanchard2017machine} first proposes a vector-wise filtering technique named Krum and raises attention to attack-robust aggregations. Afterward, dimension-wise filtering techniques are introduced, such as Median~\cite{yin2018byzantine}, Trimmed Mean~\cite{yin2018byzantine}, and signSGD based on majority voting~\cite{DBLP:conf/iclr/BernsteinZAA19}. Also, advanced vector-wise filtering methods include Multi-Krum~\cite{blanchard2017machine}, Bulyan~\cite{guerraoui2018hidden}, RFA~\cite{pillutla2019robust}, RSA~\cite{li2019rsa}, DnC~\cite{shejwalkar2021manipulating}, residual-based reweighting \cite{Fu2021}, attack-adaptive aggregation \cite{Wan2021}, and bucketing-based aggregation~\cite{DBLP:conf/iclr/KarimireddyHJ22}.
Most of these solutions can guarantee the success of defense under certain assumptions, such as IID distribution of data or the constrained portion of malicious clients. However, such assumptions do not always hold in real scenarios. 
Leveraging proxy data provides the possibility to further use server knowledge to help defend against attacks beyond idealized assumptions.
FLTrust~\cite{fltrust} maintains a server model and utilizes the statistical properties of the client model and server model to reweight the client updates. Sageflow~\cite{sageflow} combines entropy-based filtering and loss-based reweighting with the proxy data.
Both methods leverage proxy data to perform some statistical analysis to heuristically reweight the client updates,
while our method directly uses server proxy data to optimize the aggregation and leads to stabler defense performance, faster convergence, and functionality beyond solely tackling attacks such as improving FL with heterogeneous data distribution without attacks.

The related work on \textbf{Training in Subspace}~\cite{DBLP:conf/iclr/LiFLY18,gur2018gradient,vinyals2012krylov,DBLP:conf/iclr/LiFLY18,gressmann2020improving,li2022low,li2022subspace}, and \textbf{Comparison with Reweighting-based FL Works}~\cite{wu2022node,yeganeh2020inverse,xiao2021novel,nagalapatti2021game,zhang2021parameterized,fltrust,sageflow} are in \Appendix.

\section{Background}
\textbf{Generic FL.} 
Suppose we have $M$ clients with local private dataset ${\mathcal{D}}_m = \{(\vx_i, y_i)\}_{i=1}^{|\mathcal{D}_m|}$ drawn from the heterogeneous local distributions, and $\mathcal{D} = \cup_{m=1}^M \mathcal{D}_m$ denotes all data from all clients, which can be viewed as sampled from the global distribution.
Then the generic federated learning optimization problem can be formulated as
\begin{align}
    \min_{\vw}{\mathcal{L}(\vw,\mathcal{D}) ={\sum_{m=1}^{M} \alpha_{m}{\mathcal{L}_m{(\vw,\mathcal{D}_m)}}}},\label{eqn:obj}
\end{align}
where $\vw \in \mathbb{R}^d$ is the model parameter, $\alpha_m = \frac{|{\mathcal{D}}_m|}{|\mathcal{D}|}$, and $ \mathcal{L}_m{(\vw,\mathcal{D}_m)} = \frac{1}{|\mathcal{D}_m|}\sum_{\xi \in \mathcal{D}_m} \ell(\vw, \xi)$ is the empirical risk for client $m$ with $\ell(\cdot, \cdot)$ being the loss function. We denote the optimal solution of (\ref{eqn:obj}) as $\vw^*$. 

\textbf{FedAVG.} Since the data is retained by clients, the optimization problem cannot be directly solved. To approximately approach the problem, a standard solution is FedAVG~\cite{mcmahan2017communication}, which aggregates the locally trained models to a global shared model on the server. The global model $\vw^{t+1}$ is aggregated at the end of $t$-th communication round as
\begin{align}
\vw^{t+1}=\frac{1}{C^t}\sum_{m\in \mathcal{M}^t} \alpha_m{\vw_m^{t}} = \vw^{t} + \frac{1}{C^t}\sum_{m=1}^M \alpha_m \Delta_m^{t},\label{eq_FedAVG}
\end{align}
where $\mathcal{M}^t \subset [M]=\{1,2,\ldots, M\}$ is the set of clients sampled in the $t$-th round, $C^t=\sum_{m\in \mathcal{M}^t} \alpha_m$, $\vw_m^{t}$ denotes the client $m$'s local model trained with the local dataset $\mathcal{D}_m$ at the end of $t$-th communication round, and $\Delta_m^{t} = \vw_m^{t} - \vw^{t}$ denotes the cumulative local updates of client $m$ in round $t$.

\section{SmartFL}
\label{sec:method}
\subsection{Method}
In this section, we introduce SmartFL, a generic and powerful server-side aggregation strategy to smartly aggregate an optimized global model from clients' updated models using a small amount of proxy data.
Through optimizing the aggregation process in every communication round, SmartFL jointly tackles various potential challenging conditions and enables a stable and robust aggregation. 

We first introduce the \textbf{formulation of server-side optimization problem} and the straightforward full-space training scheme.
Then, we demonstrate the key component of SmartFL, i.e., \textbf{the subspace training technique}, to overcome the drawbacks of data-driven optimization on the entire model parameters. Afterward, we show the strategy for \textbf{the extension to unlabelled proxy data}. Finally, we provide the \textbf{implementation and overall process}.


\textbf{Server-side Optimization.} We aim to leverage server proxy data to optimize the global model based on the clients' local models. Note that the server-side optimization is performed on the global model $\vw^{t+1}$, for $t =0, 1, 2 \dots, T$, in the server-side aggregation process at the end of every communication round.
For simplicity, we denote the global model as $\vw$ and demonstrate the on-server optimization for the aggregation at the end of $t$-th communication round as follows.
\begin{remark}
    In our setup, we assume the server itself can collect a small amount of clean training data $\mathcal{D}_s$ (e.g., 100 samples in total) for the learning task, which is a practical assumption widely used in the prior studies~\cite{fltrust,sageflow, xiao2021novel,cheng2021fedgems,nagalapatti2021game}. 
\end{remark}

Then, we can optimize the global model $\vw$ with the empirical risk on the proxy data, denoted as $\mathcal{L}_s{(\vw, \mathcal{D}_s)}=\frac{1}{\left|\mathcal{D}_s\right|} \sum_{\xi \in \mathcal{D}_s} \ell(\boldsymbol{w}, \xi)$.

A straightforward data-driven optimization strategy is further finetuning the global model initialized with the coefficients of FedAVG, which is validated to be effective in dealing with non-IID data distribution in the prior study~\cite{chen2020fedbe} if plenty of proxy data is available. The optimization process is as follows:
\begin{align}
    &\textbf{Initialization}:\vw \leftarrow \frac{1}{C^t}\sum_{m\in \mathcal{M}^t}\alpha_m{\vw_m^t}, \\
&\textbf{Full-space Trainig}: \min_{\vw}\mathcal{L}_s{(\vw, \mathcal{D}_s)}. 
\label{eq_standard}
\end{align}
However, this strategy suffers from severe overfitting in the practical scenario, where the on-server proxy dataset is not likely to be impractically large. Also, this method can not effectively eliminate the effects of poisoning attacks.

\textbf{Subspace Training for Server-side Optimization.}
Inspired by prior studies on mode connectivity and the success of reweighting-based methods for heterogeneous/attack-robust FL, as we discussed in the introduction section, we constrain the optimization process in the promising subspace, i.e., the convex hull spanned by the clients' models. 
Instead of training the global model in the entire parameter space, we optimize the model in the reduced subspace with a significantly lower dimension.
The subspace optimization problem at the end of communication round $t$ can be formulated as 
\begin{equation}
    \begin{split}
     &\min_{\vw}\mathcal{L}_s{(\vw, \mathcal{D}_s)}, \\
      s.t.\hspace{5pt} & \vw = \sum_{m\in \mathcal{M}^t}{p_m\vw_m^{t}}, \hspace{5pt} and ~~ \vp \in \Lambda,\label{eqn:subspace-obj-0} 
\end{split}
\end{equation}
where $p_m$ is the aggregation coefficient for client $m$, and  $\mathcal{L}_s{(\vw, \mathcal{D}_s)}$ is the empirical risk, $\Lambda$ is defined as
\begin{equation}
    \begin{split}
&\Lambda = \{\vp \in \mathbb{R}^M: p_m \geq 0 \mbox{ for } m \in \mathcal{M}^t,\\
&p_m\equiv 0 \mbox{ for } m \in [M]\setminus \mathcal{M}^t, \hspace{5pt}   \sum_{m \in \mathcal{M}^t}{p_m} =1\}.
\end{split}
\end{equation}
Note that in  solving problem (\ref{eqn:subspace-obj-0}), we optimize $\vw$ over its coefficients $\vp$ with fixed $\vw_{m}^{t}$. We denote $\mathcal{L}_s{(\vw, \mathcal{D}_s)}$, with a slight abuse of notation, as $\mathcal{L}_s{(\vp, \mathcal{D}_s)}$  and the problem can then be rewritten as 
\begin{align}
  \min_{\vp \in \Lambda}\mathcal{L}_s{(\vp, \mathcal{D}_s)}. \hspace{20pt}\label{eqn:FEDST}
\end{align}
We would like to point out that all the elements in $[M]\setminus \mathcal{M}^t$ are fixed to be 0, and $\Lambda$ is essentially a $|\mathcal{M}^t|$ dimensional set. Thus, we only need to optimize $|\mathcal{M}^t|$ parameters, i.e., $p_m$ with $m\in \mathcal{M}^t$, instead of the entire neural network parameter space. This aggregation process can find the optimal model fusion, i.e., a convex combination of client models trained on non-IID datasets, by learning on the labelled proxy data. Benefiting from such a small optimization space, the generalization ability of our approach can be  significantly reinforced so that it can work well even with a small amount of proxy data. This will be further discussed in our theoretical analysis.  Moreover, when malicious clients exist, our aggregation can mitigate their negative effects by optimizing corresponding $p_m$ to small values. 




\begin{algorithm}[t]
\footnotesize
\caption{SmartFL.}
\label{alg:feddistill}
    \For {each communication round $t = 0, \dots, T$}
    {$\mathcal{M}_t \leftarrow$ selected subset of the $M$ clients
    
        \For{ each client $m \in \mathcal{M}_t$ \textbf{in parallel} }
        {
         $\vw_{m}^t \leftarrow \text{Client-LocalUpdate}(m, \vw^{t })$ 
        }
        initialize the aggregation coefficients with local samples as
        
        $\vp^{t, 0} \leftarrow \tilde{\bm{\alpha}}$ with 
        $\tilde{\alpha}_m = \alpha_m/C^t$ for all $m \in \mathcal{M}^t$ and otherwise $\tilde{\alpha}_m = 0$ 
         
        \For{ $j$ in $\{ 1, \dots, E_s \}$ }{
            update in mini-batches $\vp^{t, j} \leftarrow \mbox{proj}_\Lambda\left(\vp^{t, j-1} -\eta_s\nabla \mathcal{L}_s(\vp, \mathcal{D}_s)\right)$
        }
    $\vp^{t} \leftarrow \vp^{t, E_s}$
    
     $\vw^{t+1} \leftarrow \sum_{m \in \mathcal{M}_t}{p_m^{t}\vw_m^t}$}
    
    \Return $\vw_{T+1}$
\end{algorithm}
\textbf{Extension to Unlabelled Samples.}
Note that for the labelled proxy data, the loss function $\ell(\cdot, \cdot)$ of the empirical risk $\mathcal{L}_s\left(\boldsymbol{w}, \mathcal{D}_s\right)$ for server-side optimization is the same as the global optimization in (\ref{eqn:obj}), which is cross entropy loss in practice. 
To further facilitate the practical usage for different conditions and fairly compare with the full-space training solutions using unlabelled proxy data for heterogeneous FL~\cite{lin2020ensembleFedDF, chen2020fedbe}, we provide an extension to unlabelled samples (\textbf{SmartFL-U}). Specifically, we utilize the exact strategy in the prior work~\cite{lin2020ensembleFedDF} to generate pseudo labels with clients' ensemble logits and use Kullback-Leibler divergence loss to  drive the global model to mimic the prediction of the ensemble of client models. The only difference is that we train the model in the reduced subspace instead of full space in ~\cite{lin2020ensembleFedDF}. Since there is no quality guarantee for the pseudo labels generated from client predictions, SmartFL-U is only applied for the empirical study of handling heterogeneous data distribution. Our theoretical analysis and studies on FL with poisoning attacks focus on SmartFL with labelled proxy data.


\textbf{Implementation and Overall Process.}
Algorithm \ref{alg:feddistill} demonstrates the overall process of SmartFL. The optimization process in (\ref{eqn:FEDST}) can be solved by general  projected stochastic gradient descent algorithms~\cite{zhou2021efficient}.
\subsection{Theoretical Analysis}\label{ss_ana}
In this section, we provide a convergence property of SmartFL under poisoning attacks. Then, we show the advantages of SmartFL over FedAVG and full-space training regarding generalization capacity. Detailed description and derivations are deferred to \Appendix Sec. 2 .


\begin{property}
[\textbf{Convergence}] \label{thm:convergence}\textup{[informal]}
Assume in each server-side aggregation, there exists at least one honest client among the $M$ sampled clients. With other assumptions specified in the appendix, the expected error of SmartFL, i.e., $\mathbb{E}\left[\mathcal{L}(\vw^{T},\mathcal{D}) - \mathcal{L}(\vw^*,\mathcal{D})\right]$, can converge linearly as $T \rightarrow \infty$. 
\end{property}


\begin{remark}
     The above result shows that $SmartFL$ can converge to the optimum $\vw^*$ in the global optimization problem (\ref{eqn:obj}) efficiently even when a large number of malicious clients exist, which is consistent with our empirical results (see Section~\ref{exp-attack}). Note that in Property \ref{thm:convergence}, we allow the data on the clients to be non-IID. Therefore, this result holds naturally for the cases of non-IID data distribution without poisoning attacks.
\end{remark}

\begin{property}[\textbf{Generalization in Aggregation}]\label{thm:generalization}\textup{[informal]}
 Assume $\Lambda$ contains $|\Lambda|$ discrete choices. Denote the dataset $\mathcal{D}_s^{-1}$ generated by replacing one sample
in $\mathcal{D}_s$ with another arbitrary sample. We assume there exists $\kappa>0$, such that  $|\mathcal{L}_s(\vw, \mathcal{D}_s)-\mathcal{L}_s(\vw, \mathcal{D}_s^{-1})|\leq \kappa/|\mathcal{D}_s|$ for all $\vw$. Given the received client models $\vw_m^{t}$, $m \in \mathcal{M}^t$ in round $t$, with the probability at least $1-\delta$, the server-side aggregations $\vw_{Smart}$ of  SmartFL satisfies the generalization upper bound:
 {\small
 \begin{align}
    \mathbb{E}_{\mathcal{D}} \mathcal{L}(\vw_{Smart}, \mathcal{D}) \leq \mathcal{L}_s(\vw_{Smart}, \mathcal{D}_s) +\kappa \sqrt{\frac{\ln(2|\Lambda|/\delta)}{2|\mathcal{D}_s|}}+C, \label{eqn:generalization-bound}
 \end{align}
 where $C$ comes the domain discrepancy between $\mathcal{D}$ and $\mathcal{D}_s$, i.e.,
 \begin{align}
 C=\frac{1}{2}d_{\mathcal{H}^t\Delta\mathcal{H}^t}(\tilde{\mathcal{D}}, \tilde{\mathcal{D}}_s) + \lambda,
 \end{align}
 with $\tilde{\mathcal{D}}$ and $\tilde{\mathcal{D}}_s$ being the distribution of $\mathcal{D}$ and $\mathcal{D}_s$, $d_{\mathcal{H}\Delta\mathcal{H}}$ being the domain discrepancy between two distributions, $\lambda = \min_{\vp} \mathbb{E}_{\mathcal{D}}\mathcal{L}(\vp, \mathcal{D}) + \mathcal{L}(\vp, \mathcal{D}_s)$. $\mathcal{H}^t$ the subspace in round $t$.
 }
 

\end{property}

\begin{remark}
    The  bound in Eqn.(\ref{eqn:generalization-bound}) demonstrates that, in each aggregation,  SmartFL can generalize well because of the extremely small set $\Lambda$, which is essentially a $|\mathcal{M}^t|$-dimension space. We can also see that this upper bound is independent of the model size. In contrast, for the  generalization bound of  the full-space training approaches corresponds to Eqn.(\ref{eqn:generalization-bound}), $|\Lambda|$ should be replaced by $|\mathcal{W}|$, which would be larger than $|\Lambda|$ by lots of orders of magnitude due to the high dimension. Moreover, $C$ can be small if $\mathcal{D}$ and $\mathcal{D}_s$ are collected from two similar distributions. This verifies the superiority of  SmartFL in generalization over full-space training approaches. 
\end{remark}




\begin{table*}[t!]
    \centering
 \vspace{-0.25cm}

    \resizebox{\textwidth}{!}{%
\begin{tabular} {l|llllll|llllll}
\toprule
&\multicolumn{6}{c}{CIFAR-10} & \multicolumn{6}{c}{FMNIST}\\
 &        \multicolumn{2}{c}{$\alpha=0.01$} & \multicolumn{2}{c}{$\alpha=0.04$} & \multicolumn{2}{c}{$\alpha=0.16$}&        \multicolumn{2}{c}{$\alpha=0.01$} & \multicolumn{2}{c}{$\alpha=0.04$} & \multicolumn{2}{c}{$\alpha=0.16$}\\
Method & $C=40\%$ &        $C=20\%$ &    $C=40\%$ &            $C=20\%$ &       $C=40\%$ &       $C=20\%$ & $C=40\%$ &        $C=20\%$ &    $C=40\%$ &            $C=20\%$ &       $C=40\%$ &       $C=20\%$ \\
\midrule
FedAVG~\cite{mcmahan2017communication}    &   33.94$\pm$2.13 & 26.20$\pm$3.94 &   58.75$\pm$2.46       & 57.14$\pm$2.49& 68.77$\pm$0.75 &  70.98$\pm$0.24&   74.10$\pm$1.46 &72.56$\pm$1.55  &   87.19$\pm$2.53      & 85.15$\pm$2.49 & 90.37$\pm$0.30& 91.59$\pm$0.51\\
FedProx~\cite{li2020federated}    &   37.69$\pm$1.76 &  \underline{36.17$\pm$3.65} &   60.05$\pm$1.25      & 59.43$\pm$1.40 & 68.28$\pm$0.48 &  70.93$\pm$1.45 &    76.06$\pm$2.56 & 73.67$\pm$1.00 &   \underline{89.44$\pm$0.91}       & 86.43$\pm$1.11&  91.39$\pm$0.14& 91.72$\pm$0.42\\
Scaffold~\cite{karimireddy2019scaffold}    & 37.93$\pm$2.21 &  29.97$\pm$3.01 & 59.24$\pm$1.76      & 57.52$\pm$1.40 & 68.57$\pm$0.94& 71.06$\pm$0.47&   78.34$\pm$1.34 & 73.04$\pm$1.13 &89.01$\pm$0.96 & 86.35$\pm$1.05& 91.23$\pm$0.17& 91.19$\pm$0.30 \\
\midrule
FedDF~\cite{lin2020ensembleFedDF}$^*$     &   35.17$\pm$1.18 &   25.24$\pm$1.74 &   59.03$\pm$1.30      & 58.88$\pm$0.72 & 68.53$\pm$1.27 &  70.35$\pm$1.57&   75.24$\pm$2.35 & 74.40$\pm$0.86 &   87.48$\pm$1.86       & 87.15$\pm$2.08& 91.37$\pm$0.14 & 91.77$\pm$0.13\\
FedBE~\cite{chen2020fedbe}$^*$  &   35.97$\pm$1.76 &   26.55$\pm$3.37 &   58.27$\pm$1.11       & 58.73$\pm$3.91 & 69.06$\pm$0.70 &  70.24$\pm$1.25&   75.10$\pm$1.39 & 74.21$\pm$0.77 &   88.32$\pm$0.62       & 86.76$\pm$2.10& 91.32$\pm$0.28 & 91.73$\pm$0.23\\
\textbf{SmartFL-U$^*$}   &   40.02$\pm$1.64 &   31.21$\pm$2.87 &   60.12$\pm$0.96  & \underline{60.51$\pm$1.50} & \underline{69.90$\pm$0.66} & \underline{71.34$\pm$0.45} &   77.48$\pm$1.40 & 75.02$\pm$0.86 &   89.31$\pm$0.70       & \underline{88.01$\pm$1.93}& \underline{91.63$\pm$0.29} & \underline{91.94$\pm$0.17}\\
\midrule
ABAVG~\cite{xiao2021novel}$^\dagger$    &   35.87$\pm$2.69 &   29.93$\pm$4.89   &   \underline{61.32$\pm$2.05} & 60.49$\pm$2.59 & 69.34$\pm$0.86 &  71.02$\pm$1.07&   74.09$\pm$1.36 & 71.85$\pm$2.61 &   88.43$\pm$1.63       &86.32$\pm$2.73& 91.20$\pm$0.19 & 91.71$\pm$0.36\\
Finetuning$^\dagger$    &   \underline{44.98$\pm$2.69} &   33.60$\pm$5.23 &   60.23$\pm$0.45      & 60.29$\pm$0.73 &   68.75$\pm$0.36 & 71.10$\pm$0.49&   \underline{82.52$\pm$0.39} & \underline{80.93$\pm$0.68} &   89.21$\pm$0.29       & 87.11$\pm$1.70& 90.53$\pm$0.50 & 91.40$\pm$0.33\\
\textbf{SmartFL$^\dagger$}    &   \textbf{52.96$\pm$1.52} &   \textbf{49.97$\pm$1.46} &   \textbf{62.73$\pm$0.91}       & \textbf{63.95$\pm$0.44} & \textbf{70.28$\pm$0.45} &  \textbf{71.68$\pm$0.29}&   \textbf{83.76$\pm$0.40} & \textbf{82.77$\pm$0.11} &   \textbf{90.39$\pm$0.14}       & \textbf{90.18$\pm$0.29} &\textbf{92.01$\pm$0.38}  & \textbf{92.14$\pm$0.21}\\
\bottomrule
\end{tabular}}
        \caption{Comparison of converged test accuracy of different FL methods with ResNet-8 on CIFAR-10 and FMNIST in $T=200$ communication rounds with different degrees of data heterogeneity $\alpha$ 
 and participation rates $C$. $^*$Methods assume the availability of unlabelled proxy data. $^\dagger$Methods assume the availability of labelled proxy data.}
\label{tab:overview_1}
\end{table*}
\begin{figure*}[t]
    \begin{subfigure}{0.49\linewidth}
        \centering
        \includegraphics[width=1.0\linewidth]{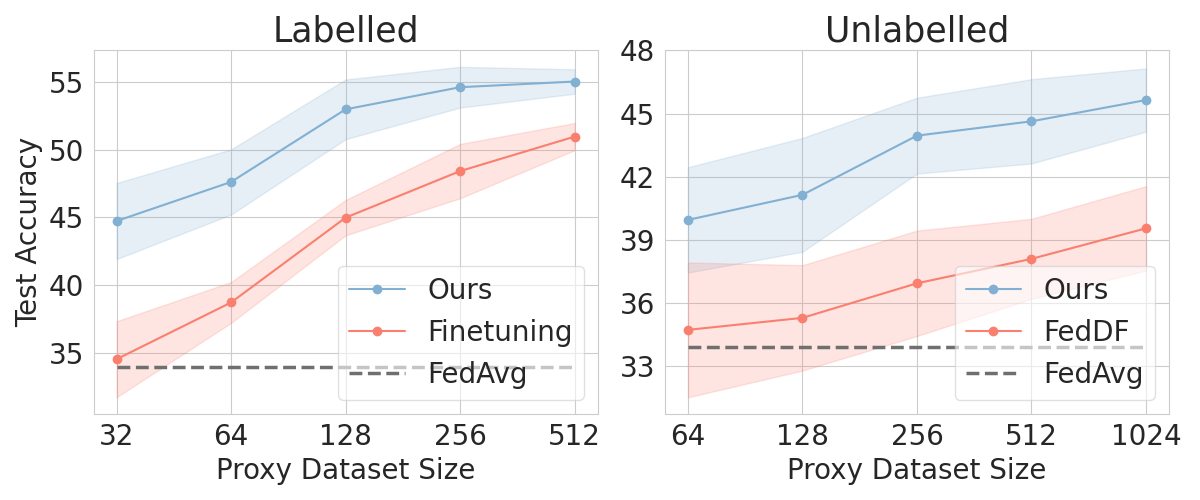}
        \subcaption{CIFAR-10}
    \end{subfigure}
        \begin{subfigure}{0.49\linewidth}
        \centering
        \includegraphics[width=1.0\linewidth]{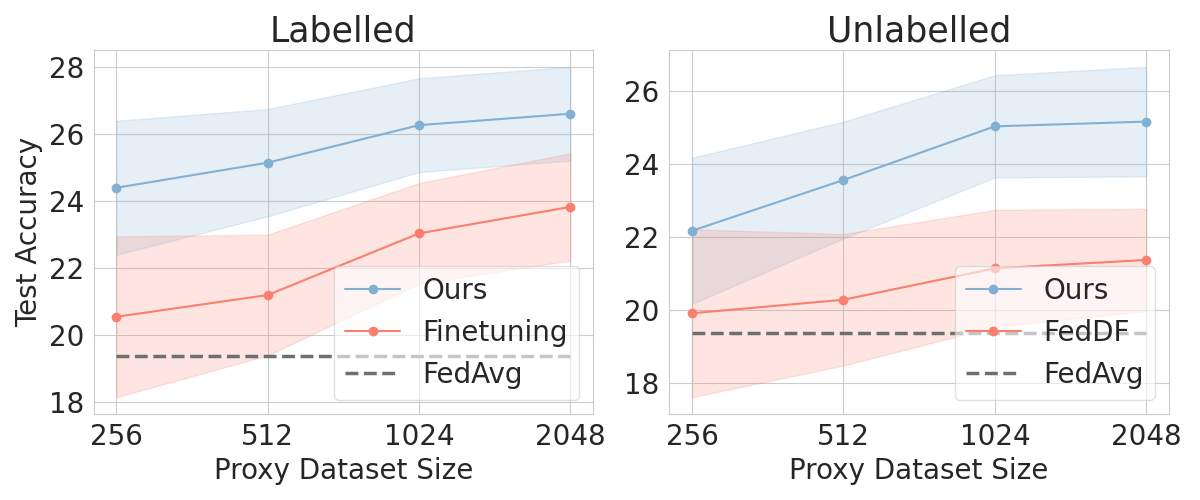}
        \subcaption{CIFAR-100}
    \end{subfigure}
    \caption{\textbf{Effects of the amount of server data}. We compare our method with Finetuning and FedDF with labelled and unlabelled server data, respectively. The horizontal axis is the size of the server data, while the vertical axis is the test accuracy. We can see that our method shows superior performances given different amounts of server data in both scenarios.}
    \label{fig:val_size}
\end{figure*}
\section{Experiments}
\subsection{Setup}
\label{ss_exp_setup}
Due to limited space, more detailed descriptions and settings can be found in \Appendix Sec. 3.

\textbf{Datasets, models, and settings.}
We consider four computer vision datasets, i.e., CIFAR-10/100~\cite{krizhevsky2009learning}, and MNIST~\cite{deng2012mnist}, FMNIST~\cite{xiao2017/online}, and extend to a NLP dataset, 20 Newsgroup~\cite{lang1995newsweeder}.
We evaluate different FL methods on the architectures of logistic regression, 2-layer ConvNet~\cite{lecun1998gradient}, MobileNet\cite{howard2017mobilenets}, ResNet-8~\cite{he2016deep} and ShuffleNet~\cite{ma2018shufflenet}. 
For the methods involving on-server data, we randomly sample 128 training samples as unlabelled/labelled proxy data on the server by default, and the others are distributed to the clients.
For other models, all the training data are distributed to clients.
Note that the numbers of total training samples for all the methods are the same.
We evaluate the FL methods with the official test set with the global model. 

\textbf{Federated learning environment.}
Similar to the prior studies~\cite{Fedaux,chen2021bridging}, we consider FL system with a practical number $n = 80$ clients with partial participation rate $C \in \{ 20\%, 40\%, 60\%\}$.  To simulate \textbf{non-IID data distributions} across clients, we follow prior studies~\cite{lin2020ensembleFedDF,chen2020fedbe} to use the Dirichlet distribution to create non-IID distribution of client training data~\cite{hsu2019measuring}. The parameter $\alpha$ controls the degree of non-IIDness.
The smaller the value of $\alpha$, the partition is closer to that one client only holds samples from a single class. Overall, we consider various non-IID degrees with $\alpha \in \{0.01, 0.04, 0.1, 0.16, 0.32, 0.64, 1\}$. For the studies involving \textbf{poisoning attacks}, we consider both data poisoning and model poisoning attacks, including Label Flip Attack~\cite{fung2018mitigating}, Omniscient Attack~\cite{blanchard2017machine}, and Fang Attack~\cite{fang2020local}.


\textbf{Baselines.}
We consider both state-of-the-art solutions against non-IID data and poisoning attacks.
For the studies on \textbf{robustness against non-IID distribution without poisoning attacks}, we include 1) without proxy data: FedAVG~\cite{mcmahan2017communication}, FedProx~\cite{li2020federated}, Scaffold~\cite{karimireddy2019scaffold}, 2) with unlabelled proxy data (i.e., FedDF~\cite{lin2020ensembleFedDF} and FedBE~\cite{chen2020fedbe}), 3) with labelled proxy data (i.e., ABAVG~\cite{xiao2021novel}) and full-space Finetuning with labelled proxy data.
For the studies on \textbf{robustness against poisoning attacks under different scenarios}, besides the applicable ones of the mentioned solutions, we further include Median~\cite{yin2018byzantine}, Krum~\cite{blanchard2017machine}, and Trimmed Mean~\cite{yin2018byzantine}, and the state-of-the-art defense with the availability of labelled proxy data, i.e., Sageflow~\cite{sageflow} and FLTrust~\cite{fltrust}.


\textbf{Local training and server aggregation setting.}
For the local training of all methods, we use the learning rate of $\eta=10^{-3}$ and the batch size of $32$ with Adam optimizer~\cite{kingma2014adam}. Local training epoch E is set to 1, and the total round is 200 by default. 
For the on-server optimization of our method, we use the batch size of 32 and Adam optimizer, fix $E_s =20$, and $\eta_s = 1e-2$ for \textbf{SmartFL} with labelled data, $\eta_s = 5e-4$ for \textbf{SmartFL-U} with unlabelled data by default.

\subsection{Robustness against Non-IID Data Distribution}
\label{s_exp_noniid}

\subsubsection{Performance Overview for Different Scenarios}
We evaluate the performance of SmartFL on widely-used benchmarks of image classification on CIFAR-10 and FMNIST under various scenarios. 
\autoref{tab:overview_1} summarizes the results. Our observations are as follows: First, FedAVG suffers from significant performance degradation when the data distribution is highly non-IID, and FedProx and Scaffold can alleviate the problem to some extent by modifying local training. Note that they are orthogonal with server-side aggregation and can be compatible with our methods. Second, leveraging a practical amount of server proxy data with advanced aggregation strategies can improve performance in most cases, indicating the potential of improving aggregation with reasonable server knowledge.
Third, for both data availability settings of labelled and unlabelled data, SmartFL and SmartFL-U consistently outperform the full-space training tuning counterpart, i.e., Finetuning and FedDF, as well as the advanced ensemble solution FedBE and heuristic reweighting solution ABAVG by a noticeable margin under various non-IIDness and participation rate settings. 
Moreover, we further demonstrate in \Appendix that SmartFL greatly accelerates convergence and requires much fewer communication rounds to achieve the target accuracy. More experiments on \textbf{20newsgroup} are shown in \Appendix Sec. 4.
Overall, the results indicate that SmartFL effectively improves the robustness of server-side aggregation against non-IID data distribution.

\subsubsection{In-depth Analysis}
\label{s_exp_indepth}
\textbf{Effect of the amount of server data.} 
We investigate the effect of the amount of server data on CIFAR-10/100, under the high level of heterogeneity with $\alpha=0.01, C=0.4$.
For labelled/unlabelled data, we compare SmartFL/SmartFL-U with FedAVG and full-space training counterpart Finetuning/FedDF. As shown in \autoref{fig:val_size}, with a reasonable amount of proxy data, 
all the optimization strategies outperform FedAVG and benefit from the increase of available data. Our solution consistently outperforms the full-space training counterparts in two datasets for both labelled and unlabelled data. This aligns with our intuition that a limited amount of proxy data can not well supervise the learning of a deep learning model with massive parameters, while our extracted subspace effectively solves the problem and enables taking advantage of even a small amount of data.

\begin{figure}[t!]
        \begin{subfigure}{.49\linewidth}
            \centering
        \includegraphics[width=1.0\linewidth]{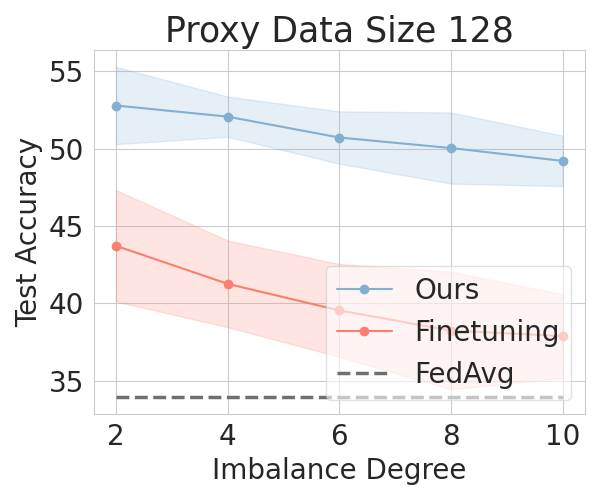}
        \end{subfigure}
        \begin{subfigure}{.49\linewidth}
        \centering
        \includegraphics[width=1.0\linewidth]{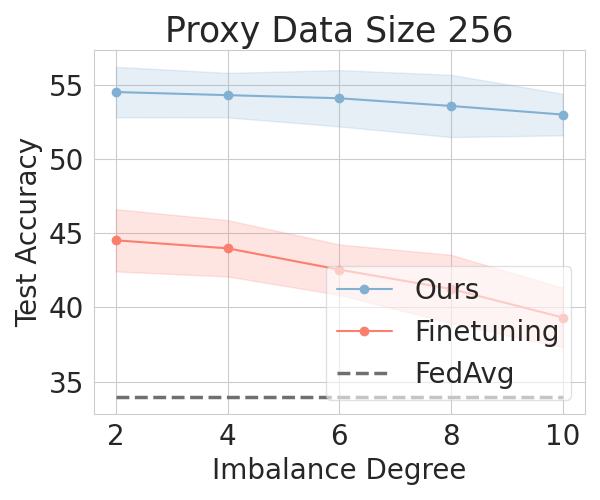}
        \end{subfigure}
\caption{\textbf{Effect of the distribution of server proxy data.} Comparison of effects of biased server proxy data on SmartFL and Finetuning. When the data imbalance degree grows, SmartFL maintains a remarkable improvement compared with FedAVG, while the improvement of Finetuning becomes slight.}
\label{fig:imbalance}
\end{figure}

\begin{figure}[t!]
        \begin{subfigure}{.49\linewidth}
            \centering
        \includegraphics[width=1.0\linewidth]{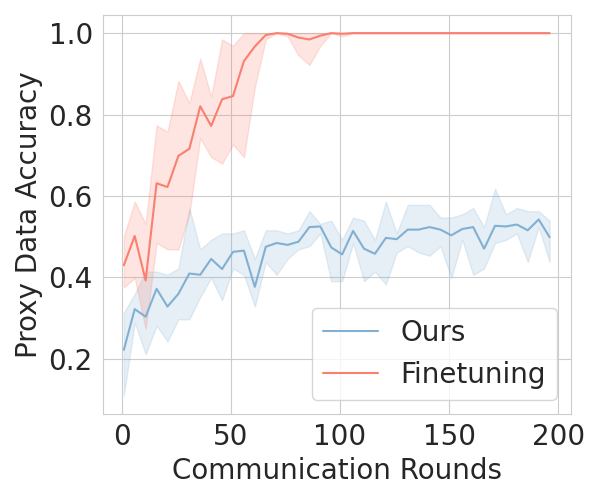}
            \subcaption{Proxy Data Accuracy}
        \end{subfigure}
        \begin{subfigure}{.49\linewidth}
        \centering
        \includegraphics[width=1.0\linewidth]{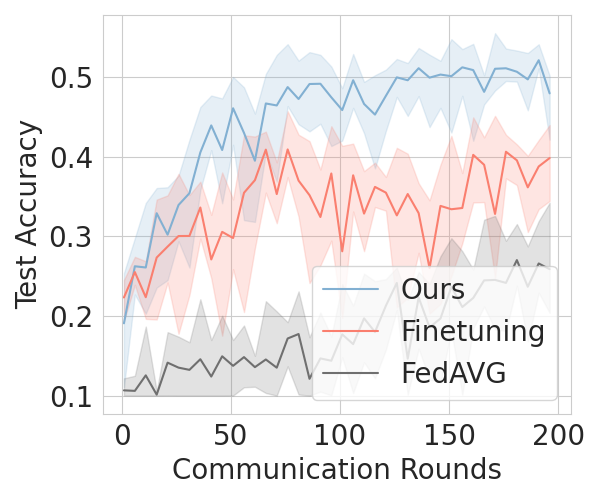}
        \subcaption{Test Accuracy}
        \end{subfigure}
        \caption{\textbf{Generalization ability with a small amount of server data}. We compare SmartFL with Finetuning, which directly finetunes the model parameters with the proxy data in every communication round. We can observe that Finetuning quickly reaches 100\% accuracy on the server data, while our method prevents overfitting and consistently demonstrate better test performance.}
    \label{fig:exp-overfit}
\end{figure}
\textbf{Effect of the distribution of server data.}
We study the influence of proxy data distribution. \autoref{fig:imbalance} shows the performance with various degrees of gap between server proxy data and global distribution with ResNet8 on CIFAR-10. Here we apply the strategy in the work~\cite{shu2019meta} to use the imbalance degree calculated with the maximum class sample number divided by the minimum class sample number. The higher the imbalance degree, the larger the distribution discrepancy between proxy data and global data with balanced classes. By optimizing the global model in a constrained optimization space, SmartFL boosts performance even with proxy data with a highly-different distribution, which further verifies its feasibility in real-world scenarios.


\begin{table}
\tiny
    \centering
 \vspace{-0.25cm}
\begin{scriptsize}
\begin{tabular}{ccccc}

\toprule
 &        \multicolumn{2}{c}{$\alpha=0.01$} & \multicolumn{2}{c}{$\alpha=0.04$}\\
Method &   5 epochs  & 10 epochs  &   5 epochs  & 10 epochs \\
\midrule
FedAVG     &  19.97$\pm$4.62 & 18.30$\pm$2.62 &   44.42$\pm$4.69    & 36.80$\pm$1.58 \\
FedProx   &  \underline{33.31$\pm$1.25} & \underline{30.16$\pm$1.89} &   48.76$\pm$0.55    & \underline{45.58$\pm$1.37} \\
Scaffold    &   31.06$\pm$1.30 & 29.85$\pm$1.53 &   46.35$\pm$0.38    & 43.06$\pm$1.24 \\
\midrule
FedDF$^*$     & 27.43$\pm$1.72 & 24.83$\pm$2.99  &   45.73$\pm$1.12    & 38.14$\pm$0.81 \\
FedBE$^*$   & 22.95$\pm$3.84 & 21.20$\pm$4.01&   45.25$\pm$2.01    & 37.36$\pm$0.98 \\
\textbf{SmartFL-U}$^*$    & 31.13$\pm$3.83 & 28.35$\pm$2.07 &   \underline{48.78$\pm$0.85}    & 41.12$\pm$2.35 \\
\midrule
ABAVG$^\dagger$   &  26.53$\pm$3.97 & 23.75$\pm$5.11 &   44.54$\pm$0.59    & 38.62$\pm$0.57 \\
Finetuning$^\dagger$   &  25.08$\pm$4.90 & 23.00$\pm$3.61 &   44.20$\pm$2.48    & 38.47$\pm$2.34 \\
\textbf{SmartFL$^\dagger$}    &   \textbf{48.63$\pm$1.19} & \textbf{50.40$\pm$0.56} &   \textbf{60.03$\pm$1.05}    & \textbf{60.33$\pm$0.65} \\ 
\bottomrule
\end{tabular}
\caption{\textbf{Effect of local epochs.} Comparison of top-1 test accuracy achieved by different FL methods with Resnet-8 on CIFAR-10 under various degrees of data heterogeneity and local training epochs. SmartFL maintains high effectiveness.}
\label{tab:local_epoch}
\end{scriptsize}
\end{table}

\textbf{Generalization ability.}
We then empirically demonstrate the generalization ability with a small amount of server data by comparing SmartFL with the full-space training counterpart, i.e., Finetuning, on CIFAR-10.
As shown in \autoref{fig:exp-overfit}, for the Finetuning approach, even though we only finetune the aggregated model for one epoch on the server at each round to try to eliminate overfitting, the accuracy calculated over the proxy data still converges to 100\% after multiple rounds, while the test accuracy does not boost significantly. On the other hand, though our method does not achieve perfect proxy data accuracy, its test performance consistently surpasses the Finetuning, which verifies that subspace training enables SmartFL to be less prone to overfitting the proxy data and boost the test performance.

\textbf{Effect of local training epochs.} 
We explore the effect of different local epochs (\autoref{tab:local_epoch}). For local epochs 5 and 10, we set the number of communication rounds to 100 and 50, respectively.  Generally, when the local epoch increases, the diversity of local updates increases, which leads to degraded performance for various FL methods. Compared with the baselines, SmartFL can alleviate the problem and maintain robustness against data heterogeneity for different local epochs.  
\begin{figure*}[t!]
    \centering
\includegraphics[width=0.95\linewidth]{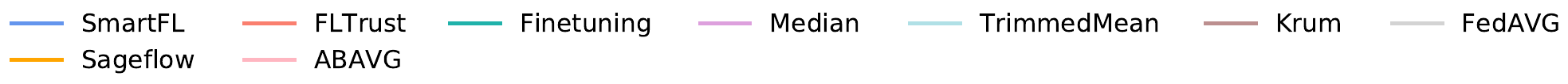}
    \begin{subfigure}{1.\linewidth}
        \begin{subfigure}{0.24\linewidth}
        \centering
        \includegraphics[width=1.0\linewidth]{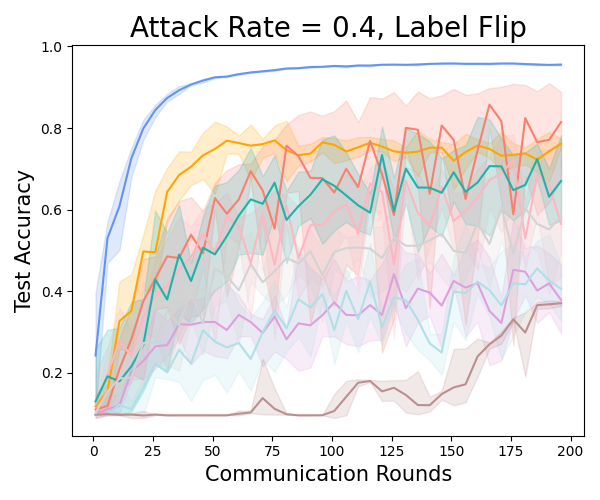}
        \end{subfigure}
        \begin{subfigure}{0.24\linewidth}
        \centering
        \includegraphics[width=1.0\linewidth]{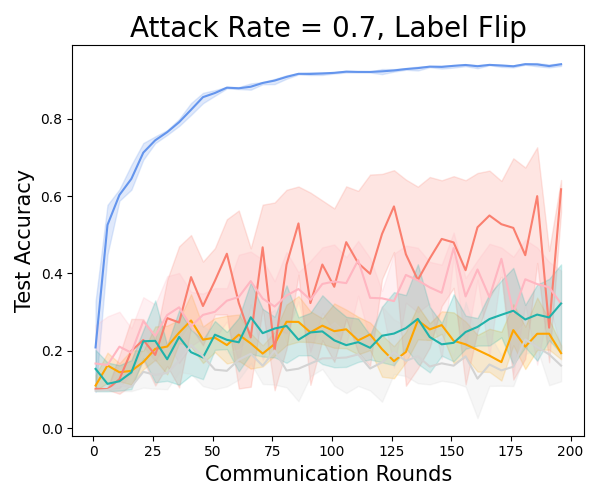}

    \end{subfigure}
    \begin{subfigure}{0.24\linewidth}
        \centering
        \includegraphics[width=1.0\linewidth]{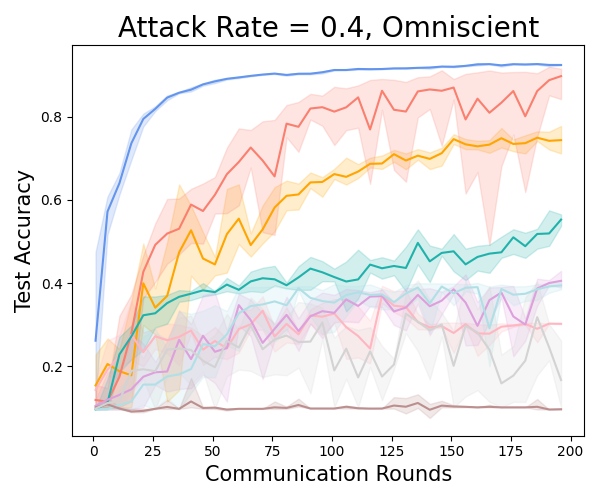}
        \end{subfigure}
    \begin{subfigure}{0.24\linewidth}
        \centering
        \includegraphics[width=1.0\linewidth]{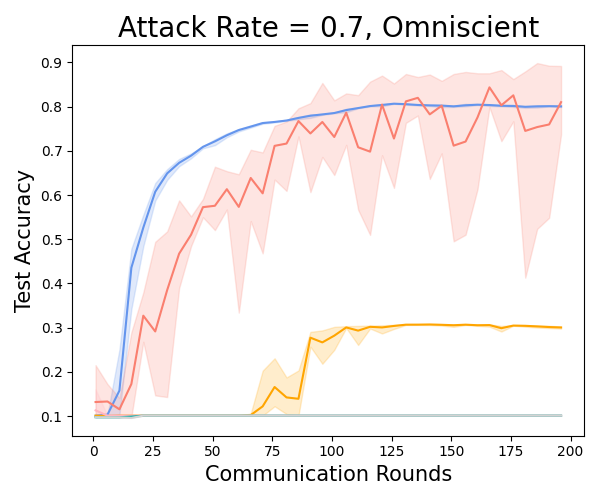} 
    \end{subfigure}
    \subcaption{MNIST}
    \end{subfigure}

    \begin{subfigure}{1.\linewidth}
    \begin{subfigure}{0.24\linewidth}
        \centering
        \includegraphics[width=1.0\linewidth]{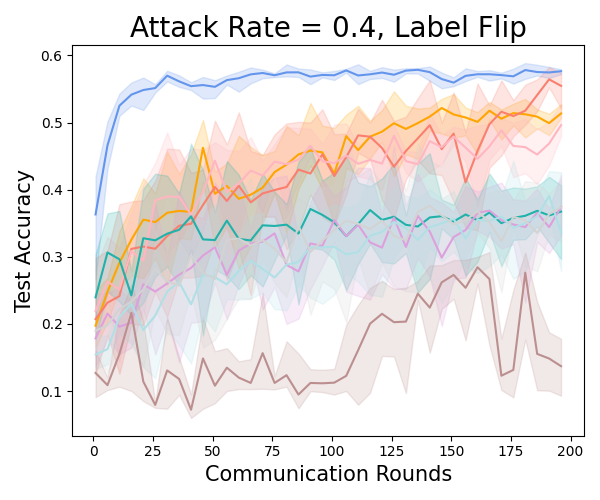}
    \end{subfigure}
        \begin{subfigure}{0.24\linewidth}
        \centering
        \includegraphics[width=1.0\linewidth]{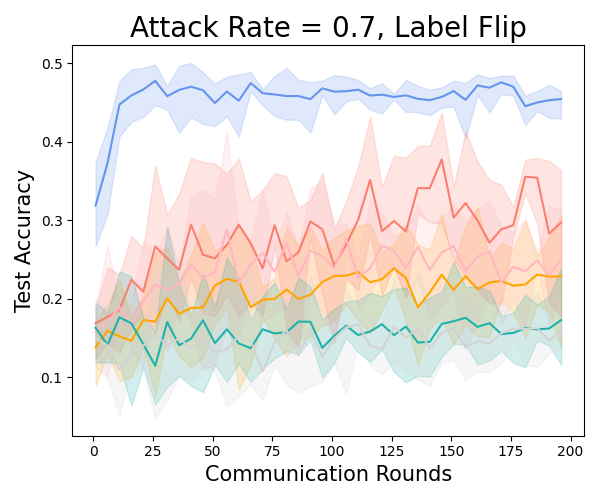}

    \end{subfigure}
    \begin{subfigure}{0.24\linewidth}
        \centering
        \includegraphics[width=1.0\linewidth]{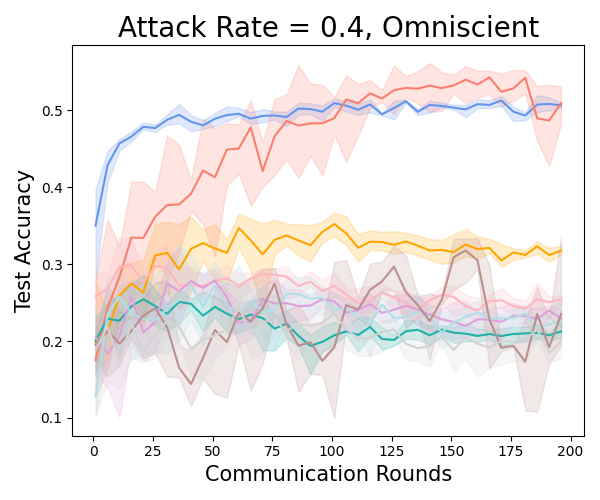}
    \end{subfigure}
    \begin{subfigure}{0.24\linewidth}
        \centering
        \includegraphics[width=1.0\linewidth]{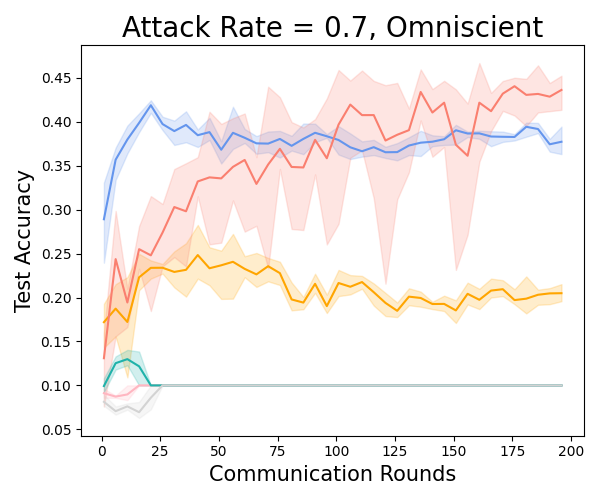} 
    \end{subfigure}
    \subcaption{CIFAR-10}
    \end{subfigure}
    \caption{\textbf{Defence against attack.} Test accuracy curve for different FL methods on MNIST/CIFAR-10 with the degree of data heterogeneity $\alpha = 0.01$ and $\alpha = 0.1$, respectively, under different types of attacks (Label Flip and Omniscient Attack) and different attack rates $AR = 0.4/0.7$. SmartFL outperforms both state-of-the-art defenses using proxy data and classical byzantine-robust aggregation strategies. }
    \label{fig:exp-attack}
\end{figure*}
\setlength{\belowcaptionskip}{-0.1cm}  
\begin{figure}[t!]
        \begin{subfigure}{.49\linewidth}
            \centering
        \includegraphics[width=1.0\linewidth]{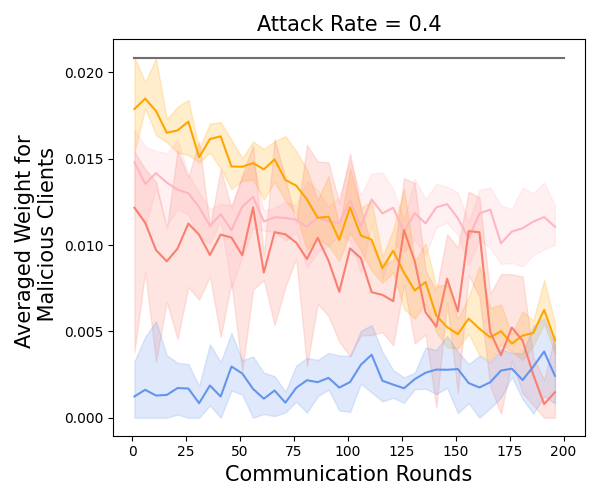}
        \end{subfigure}
        \begin{subfigure}{.49\linewidth}
        \centering
        \includegraphics[width=1.0\linewidth]{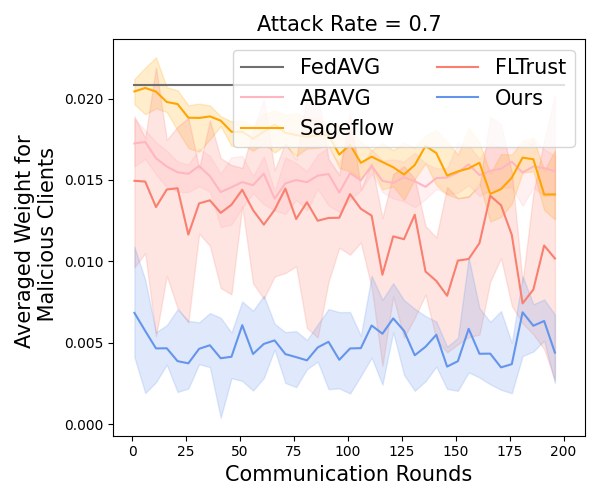}
        \end{subfigure}
\caption{\textbf{Curve of the average of coefficients $p$ for malicious clients} in every communication round with Label Flip Attack on CIFAR-10 with $\alpha=0.1$. SmartFL successfully learns small weights for malicious clients during the whole training process, which enables robust FL even when the data is highly non-IID, and the attack rate is large.}
\label{fig:visualization}
\end{figure}
\subsection{Robustness against Poisoning Attacks}
\label{exp-attack}
We demonstrate the robustness of our solution against Label Flip Attack, Omniscient Attack, and Fang Attack in various scenarios. 
We experiment on CIFAR 10 with $\alpha \in \{0.1,1\}, C = 60\%$ and model ResNet-8, and MNIST with $\alpha \in \{0.01, 0.1, 1\}, C = 60\%$ and model 2-layer ConvNet, and consider attack rate $AR \in \{0.2,0.3,0.4,0.7\}$. 
\autoref{fig:exp-attack} demonstrates the server test accuracy in the federated learning process with attack rate $AR=0.4/0.7$ and high data heterogeneity. More results in all the scenarios are shown in \Appendix Sec. 4.
For Attack Rate $=0.7$, some defenses are not applicable because their assumption that less than half of the clients are malicious does not hold. 

We have the following observations. First, update-based statistical solutions generally cannot perform well when the data distribution is highly non-IID, which aligns with prior studies~\cite{DBLP:conf/iclr/KarimireddyHJ22,fltrust}, indicating the potential to leverage additional server knowledge to further improve the robustness. Second, full-space training after performing FedAVG, i.e., Finetuning, is hard to dilute the influence of poisoned models with a small amount of server data. 
Third, the state-of-the-art methods using labelled proxy data (i.e., Sageflow and FLTrust) show a relatively good performance defending against both attacks but still suffer from unstable learning and some failure cases.
Finally, SmartFL yields stable and good performance against various attacks in different scenarios, indicating the effectiveness of mitigating negative effects from malicious clients through subspace training. 

We further visualize \textbf{the evolution of averaged coefficients for malicious clients} of competitive attack-resistant FL methods and FedAVG (\autoref{fig:visualization}). Note that the weights are normalized for all solutions, and small weights for malicious clients result in a larger contribution of useful benign updates in the global model.  We can observe that SmartFL successfully gives small weights to malicious clients in the whole learning process, while the other solutions take effect when the model is well-trained. This accounts for the faster convergence and higher performance of SmartFL. 



\section{Conclusion and Discussions}
\label{sec-concu}
Data heterogeneity across clients and poisoning attacks are among the main bottlenecks for robust server-side aggregation. In this work, we propose SmartFL, which optimizes the aggregation to universally overcome both challenges by subspace training. We extract a reduced subspace spanned by the clients' models to achieve effective and efficient optimization of the global model in every communication round with a small amount of proxy data. We provide theoretical analysis for SmartFL on convergence and generalization ability. Extensive experiments demonstrate the state-of-the-art performance of SmartFL for both FL with non-IID data distribution and FL with poisoning attacks.
We involve more discussions in \Appendix Sec. 5.
{\small
\bibliographystyle{ieee_fullname}
\bibliography{egbib}

\begin{thebibliography}{10}\itemsep=-1pt

\bibitem{Acar2021Dyn}
Durmus Alp~Emre Acar, Yue Zhao, Ramon~Matas Navarro, Matthew Mattina, Paul~N.
  Whatmough, and Venkatesh Saligrama.
\newblock Federated learning based on dynamic regularization.
\newblock In {\em ICLR}, 2021.

\bibitem{ben2010theory}
Shai Ben-David, John Blitzer, Koby Crammer, Alex Kulesza, Fernando Pereira, and
  Jennifer~Wortman Vaughan.
\newblock A theory of learning from different domains.
\newblock {\em Machine learning}, 79(1):151--175, 2010.

\bibitem{DBLP:conf/iclr/BernsteinZAA19}
Jeremy Bernstein, Jiawei Zhao, Kamyar Azizzadenesheli, and Anima Anandkumar.
\newblock signsgd with majority vote is communication efficient and fault
  tolerant.
\newblock In {\em 7th International Conference on Learning Representations,
  {ICLR} 2019, New Orleans, LA, USA, May 6-9, 2019}. OpenReview.net, 2019.

\bibitem{blanchard2017machine}
Peva Blanchard, El~Mahdi El~Mhamdi, Rachid Guerraoui, and Julien Stainer.
\newblock Machine learning with adversaries: Byzantine tolerant gradient
  descent.
\newblock {\em Advances in Neural Information Processing Systems}, 30, 2017.

\bibitem{fltrust}
Xiaoyu Cao, Minghong Fang, Jia Liu, and Neil~Zhenqiang Gong.
\newblock Fltrust: Byzantine-robust federated learning via trust bootstrapping.
\newblock In {\em 28th Annual Network and Distributed System Security
  Symposium, {NDSS} 2021, virtually, February 21-25, 2021}. The Internet
  Society, 2021.

\bibitem{chen2020fedbe}
Hong-You Chen and Wei-Lun Chao.
\newblock Fedbe: Making bayesian model ensemble applicable to federated
  learning.
\newblock In {\em ICLR}, 2021.

\bibitem{chen2021bridging}
Hong-You Chen and Wei-Lun Chao.
\newblock On bridging generic and personalized federated learning for image
  classification.
\newblock In {\em International Conference on Learning Representations}, 2021.

\bibitem{cheng2021fedgems}
Sijie Cheng, Jingwen Wu, Yanghua Xiao, and Yang Liu.
\newblock Fedgems: Federated learning of larger server models via selective
  knowledge fusion.
\newblock {\em arXiv preprint arXiv:2110.11027}, 2021.

\bibitem{cho2022heterogeneous}
Yae~Jee Cho, Andre Manoel, Gauri Joshi, Robert Sim, and Dimitrios Dimitriadis.
\newblock Heterogeneous ensemble knowledge transfer for training large models
  in federated learning.
\newblock In {\em IJCAI}, 2022.

\bibitem{deng2012mnist}
Li Deng.
\newblock The mnist database of handwritten digit images for machine learning
  research.
\newblock {\em IEEE Signal Processing Magazine}, 29(6):141--142, 2012.

\bibitem{draxler2018essentially}
Felix Draxler, Kambis Veschgini, Manfred Salmhofer, and Fred Hamprecht.
\newblock Essentially no barriers in neural network energy landscape.
\newblock In {\em International conference on machine learning}, pages
  1309--1318. PMLR, 2018.

\bibitem{fang2020local}
Minghong Fang, Xiaoyu Cao, Jinyuan Jia, and Neil Gong.
\newblock Local model poisoning attacks to $\{$Byzantine-Robust$\}$ federated
  learning.
\newblock In {\em 29th USENIX Security Symposium (USENIX Security 20)}, pages
  1605--1622, 2020.

\bibitem{Fu2021}
Shuhao Fu, Chulin Xie, Bo Li, and Qifeng Chen.
\newblock Attack-resistant federated learning with residual-based reweighting.
\newblock In {\em AAAI Workshops}, 2021.

\bibitem{fung2018mitigating}
Clement Fung, Chris~JM Yoon, and Ivan Beschastnikh.
\newblock Mitigating sybils in federated learning poisoning.
\newblock {\em arXiv preprint arXiv:1808.04866}, 2018.

\bibitem{garipov2018loss}
Timur Garipov, Pavel Izmailov, Dmitrii Podoprikhin, Dmitry~P Vetrov, and
  Andrew~G Wilson.
\newblock Loss surfaces, mode connectivity, and fast ensembling of dnns.
\newblock {\em Advances in neural information processing systems}, 31, 2018.

\bibitem{gressmann2020improving}
Frithjof Gressmann, Zach Eaton-Rosen, and Carlo Luschi.
\newblock Improving neural network training in low dimensional random bases.
\newblock {\em Advances in Neural Information Processing Systems},
  33:12140--12150, 2020.

\bibitem{Fedaux}
Hang Gu, Bin Guo, Jiangtao Wang, Wen Sun, Jiaqi Liu, Sicong Liu, and Zhiwen Yu.
\newblock Fedaux: An efficient framework for hybrid federated learning.
\newblock In {\em {IEEE} International Conference on Communications, {ICC}
  2022, Seoul, Korea, May 16-20, 2022}, pages 195--200. {IEEE}, 2022.

\bibitem{guerraoui2018hidden}
Rachid Guerraoui, S{\'e}bastien Rouault, et~al.
\newblock The hidden vulnerability of distributed learning in byzantium.
\newblock In {\em International Conference on Machine Learning}, pages
  3521--3530. PMLR, 2018.

\bibitem{gur2018gradient}
Guy Gur-Ari, Daniel~A Roberts, and Ethan Dyer.
\newblock Gradient descent happens in a tiny subspace.
\newblock {\em arXiv preprint arXiv:1812.04754}, 2018.

\bibitem{hanzely2020lower}
Filip Hanzely, Slavom{\'\i}r Hanzely, Samuel Horv{\'a}th, and Peter
  Richt{\'a}rik.
\newblock Lower bounds and optimal algorithms for personalized federated
  learning.
\newblock {\em Advances in Neural Information Processing Systems},
  33:2304--2315, 2020.

\bibitem{he2016deep}
Kaiming He, Xiangyu Zhang, Shaoqing Ren, and Jian Sun.
\newblock Deep residual learning for image recognition.
\newblock In {\em Proceedings of the IEEE conference on computer vision and
  pattern recognition}, pages 770--778, 2016.

\bibitem{howard2017mobilenets}
Andrew~G Howard, Menglong Zhu, Bo Chen, Dmitry Kalenichenko, Weijun Wang,
  Tobias Weyand, Marco Andreetto, and Hartwig Adam.
\newblock Mobilenets: Efficient convolutional neural networks for mobile vision
  applications.
\newblock {\em arXiv preprint arXiv:1704.04861}, 2017.

\bibitem{hsu2019measuring}
Tzu-Ming~Harry Hsu, Hang Qi, and Matthew Brown.
\newblock Measuring the effects of non-identical data distribution for
  federated visual classification.
\newblock {\em arXiv preprint arXiv:1909.06335}, 2019.

\bibitem{kairouz2021advances}
Peter Kairouz, H~Brendan McMahan, Brendan Avent, Aur{\'e}lien Bellet, Mehdi
  Bennis, Arjun~Nitin Bhagoji, Kallista Bonawitz, Zachary Charles, Graham
  Cormode, Rachel Cummings, et~al.
\newblock Advances and open problems in federated learning.
\newblock {\em Foundations and Trends{\textregistered} in Machine Learning},
  14(1--2):1--210, 2021.

\bibitem{DBLP:conf/iclr/KarimireddyHJ22}
Sai~Praneeth Karimireddy, Lie He, and Martin Jaggi.
\newblock Byzantine-robust learning on heterogeneous datasets via bucketing.
\newblock In {\em The Tenth International Conference on Learning
  Representations, {ICLR} 2022, Virtual Event, April 25-29, 2022}.
  OpenReview.net, 2022.

\bibitem{karimireddy2019scaffold}
Sai~Praneeth Karimireddy, Satyen Kale, Mehryar Mohri, Sashank~J Reddi,
  Sebastian~U Stich, and Ananda~Theertha Suresh.
\newblock Scaffold: Stochastic controlled averaging for on-device federated
  learning.
\newblock In {\em ICML}, 2020.

\bibitem{kingma2014adam}
Diederik~P Kingma and Jimmy Ba.
\newblock Adam: A method for stochastic optimization.
\newblock {\em arXiv preprint arXiv:1412.6980}, 2014.

\bibitem{konevcny2016federated}
Jakub Kone{\v{c}}n{\`y}, H~Brendan McMahan, Felix~X Yu, Peter Richt{\'a}rik,
  Ananda~Theertha Suresh, and Dave Bacon.
\newblock Federated learning: Strategies for improving communication
  efficiency.
\newblock {\em arXiv preprint arXiv:1610.05492}, 2016.

\bibitem{krizhevsky2009learning}
Alex Krizhevsky, Geoffrey Hinton, et~al.
\newblock Learning multiple layers of features from tiny images.
\newblock 2009.

\bibitem{kuditipudi2019explaining}
Rohith Kuditipudi, Xiang Wang, Holden Lee, Yi Zhang, Zhiyuan Li, Wei Hu, Rong
  Ge, and Sanjeev Arora.
\newblock Explaining landscape connectivity of low-cost solutions for
  multilayer nets.
\newblock {\em Advances in neural information processing systems}, 32, 2019.

\bibitem{kulkarni2020survey}
Viraj Kulkarni, Milind Kulkarni, and Aniruddha Pant.
\newblock Survey of personalization techniques for federated learning.
\newblock In {\em 2020 Fourth World Conference on Smart Trends in Systems,
  Security and Sustainability (WorldS4)}, pages 794--797. IEEE, 2020.

\bibitem{lang1995newsweeder}
Ken Lang.
\newblock Newsweeder: Learning to filter netnews.
\newblock In {\em Machine Learning Proceedings 1995}, pages 331--339. Elsevier,
  1995.

\bibitem{lecun1998gradient}
Yann LeCun, L{\'e}on Bottou, Yoshua Bengio, and Patrick Haffner.
\newblock Gradient-based learning applied to document recognition.
\newblock {\em Proceedings of the IEEE}, 86(11):2278--2324, 1998.

\bibitem{DBLP:conf/iclr/LiFLY18}
Chunyuan Li, Heerad Farkhoor, Rosanne Liu, and Jason Yosinski.
\newblock Measuring the intrinsic dimension of objective landscapes.
\newblock In {\em 6th International Conference on Learning Representations,
  {ICLR} 2018, Vancouver, BC, Canada, April 30 - May 3, 2018, Conference Track
  Proceedings}. OpenReview.net, 2018.

\bibitem{li2019fedmd}
Daliang Li and Junpu Wang.
\newblock Fedmd: Heterogenous federated learning via model distillation.
\newblock {\em arXiv preprint arXiv:1910.03581}, 2019.

\bibitem{li2019rsa}
Liping Li, Wei Xu, Tianyi Chen, Georgios~B Giannakis, and Qing Ling.
\newblock Rsa: Byzantine-robust stochastic aggregation methods for distributed
  learning from heterogeneous datasets.
\newblock In {\em Proceedings of the AAAI Conference on Artificial
  Intelligence}, volume~33, pages 1544--1551, 2019.

\bibitem{li2021model}
Qinbin Li, Bingsheng He, and Dawn Song.
\newblock Model-contrastive federated learning.
\newblock In {\em Proceedings of the IEEE/CVF Conference on Computer Vision and
  Pattern Recognition}, pages 10713--10722, 2021.

\bibitem{li2021ditto}
Tian Li, Shengyuan Hu, Ahmad Beirami, and Virginia Smith.
\newblock Ditto: Fair and robust federated learning through personalization.
\newblock In {\em International Conference on Machine Learning}, pages
  6357--6368. PMLR, 2021.

\bibitem{li2020federated}
Tian Li, Anit~Kumar Sahu, Manzil Zaheer, Maziar Sanjabi, Ameet Talwalkar, and
  Virginia Smith.
\newblock Federated optimization in heterogeneous networks.
\newblock In {\em MLSys}, 2020.

\bibitem{li2022low}
Tao Li, Lei Tan, Zhehao Huang, Qinghua Tao, Yipeng Liu, and Xiaolin Huang.
\newblock Low dimensional trajectory hypothesis is true: Dnns can be trained in
  tiny subspaces.
\newblock {\em IEEE Transactions on Pattern Analysis and Machine Intelligence},
  2022.

\bibitem{li2022subspace}
Tao Li, Yingwen Wu, Sizhe Chen, Kun Fang, and Xiaolin Huang.
\newblock Subspace adversarial training.
\newblock In {\em Proceedings of the IEEE/CVF Conference on Computer Vision and
  Pattern Recognition}, pages 13409--13418, 2022.

\bibitem{lin2020ensembleFedDF}
Tao Lin, Lingjing Kong, Sebastian~U Stich, and Martin Jaggi.
\newblock Ensemble distillation for robust model fusion in federated learning.
\newblock {\em Advances in Neural Information Processing Systems},
  33:2351--2363, 2020.

\bibitem{ma2018shufflenet}
Ningning Ma, Xiangyu Zhang, Hai-Tao Zheng, and Jian Sun.
\newblock Shufflenet v2: Practical guidelines for efficient cnn architecture
  design.
\newblock In {\em Proceedings of the European conference on computer vision
  (ECCV)}, pages 116--131, 2018.

\bibitem{mcmahan2017communication}
H~Brendan McMahan, Eider Moore, Daniel Ramage, Seth Hampson, et~al.
\newblock Communication-efficient learning of deep networks from decentralized
  data.
\newblock In {\em AISTATS}, 2017.

\bibitem{nagalapatti2021game}
Lokesh Nagalapatti and Ramasuri Narayanam.
\newblock Game of gradients: Mitigating irrelevant clients in federated
  learning.
\newblock In {\em Proceedings of the AAAI Conference on Artificial
  Intelligence}, volume~35, pages 9046--9054, 2021.

\bibitem{oh2020mix2fld}
Seungeun Oh, Jihong Park, Eunjeong Jeong, Hyesung Kim, Mehdi Bennis, and
  Seong-Lyun Kim.
\newblock Mix2fld: Downlink federated learning after uplink federated
  distillation with two-way mixup.
\newblock {\em IEEE Communications Letters}, 24(10):2211--2215, 2020.

\bibitem{sageflow}
Jungwuk Park, Dong{-}Jun Han, Minseok Choi, and Jaekyun Moon.
\newblock Sageflow: Robust federated learning against both stragglers and
  adversaries.
\newblock In Marc'Aurelio Ranzato, Alina Beygelzimer, Yann~N. Dauphin, Percy
  Liang, and Jennifer~Wortman Vaughan, editors, {\em Advances in Neural
  Information Processing Systems 34: Annual Conference on Neural Information
  Processing Systems 2021, NeurIPS 2021, December 6-14, 2021, virtual}, pages
  840--851, 2021.

\bibitem{pillutla2019robust}
Krishna Pillutla, Sham~M Kakade, and Zaid Harchaoui.
\newblock Robust aggregation for federated learning.
\newblock {\em arXiv preprint arXiv:1912.13445}, 2019.

\bibitem{polyak1964gradient}
Boris~T Polyak.
\newblock Gradient methods for solving equations and inequalities.
\newblock {\em USSR Computational Mathematics and Mathematical Physics},
  4(6):17--32, 1964.

\bibitem{rieke2020future}
Nicola Rieke, Jonny Hancox, Wenqi Li, Fausto Milletari, Holger~R Roth, Shadi
  Albarqouni, Spyridon Bakas, Mathieu~N Galtier, Bennett~A Landman, Klaus
  Maier-Hein, et~al.
\newblock The future of digital health with federated learning.
\newblock {\em NPJ digital medicine}, 3(1):1--7, 2020.

\bibitem{shejwalkar2021manipulating}
Virat Shejwalkar and Amir Houmansadr.
\newblock Manipulating the byzantine: Optimizing model poisoning attacks and
  defenses for federated learning.
\newblock In {\em NDSS}, 2021.

\bibitem{poisoning_attack}
Virat Shejwalkar, Amir Houmansadr, Peter Kairouz, and Daniel Ramage.
\newblock Back to the drawing board: A critical evaluation of poisoning attacks
  on production federated learning.
\newblock In {\em 2022 IEEE Symposium on Security and Privacy (SP)}, pages
  1354--1371, 2022.

\bibitem{shin2020xor}
MyungJae Shin, Chihoon Hwang, Joongheon Kim, Jihong Park, Mehdi Bennis, and
  Seong-Lyun Kim.
\newblock Xor mixup: Privacy-preserving data augmentation for one-shot
  federated learning.
\newblock {\em arXiv preprint arXiv:2006.05148}, 2020.

\bibitem{shu2019meta}
Jun Shu, Qi Xie, Lixuan Yi, Qian Zhao, Sanping Zhou, Zongben Xu, and Deyu Meng.
\newblock Meta-weight-net: Learning an explicit mapping for sample weighting.
\newblock {\em Advances in neural information processing systems}, 32, 2019.

\bibitem{t2020personalized}
Canh T~Dinh, Nguyen Tran, and Josh Nguyen.
\newblock Personalized federated learning with moreau envelopes.
\newblock {\em Advances in Neural Information Processing Systems},
  33:21394--21405, 2020.

\bibitem{vinyals2012krylov}
Oriol Vinyals and Daniel Povey.
\newblock Krylov subspace descent for deep learning.
\newblock In {\em Artificial intelligence and statistics}, pages 1261--1268.
  PMLR, 2012.

\bibitem{wainwright2019high}
Martin~J Wainwright.
\newblock {\em High-dimensional statistics: A non-asymptotic viewpoint},
  volume~48.
\newblock Cambridge University Press, 2019.

\bibitem{Wan2021}
Ching~Pui Wan and Qifeng Chen.
\newblock Robust federated learning with attack-adaptive aggregation.
\newblock In {\em IJCAI Workshops}, 2021.

\bibitem{wang2020tackling}
Jianyu Wang, Qinghua Liu, Hao Liang, Gauri Joshi, and H~Vincent Poor.
\newblock Tackling the objective inconsistency problem in heterogeneous
  federated optimization.
\newblock {\em Advances in neural information processing systems},
  33:7611--7623, 2020.

\bibitem{wu2022node}
Hongda Wu and Ping Wang.
\newblock Node selection toward faster convergence for federated learning on
  non-iid data.
\newblock {\em IEEE Transactions on Network Science and Engineering}, 2022.

\bibitem{xiao2017/online}
Han Xiao, Kashif Rasul, and Roland Vollgraf.
\newblock Fashion-mnist: a novel image dataset for benchmarking machine
  learning algorithms, 2017.

\bibitem{xiao2021novel}
Jianhang Xiao, Chunhui Du, Zijing Duan, and Wei Guo.
\newblock A novel server-side aggregation strategy for federated learning in
  non-iid situations.
\newblock In {\em 2021 20th International Symposium on Parallel and Distributed
  Computing (ISPDC)}, pages 17--24. IEEE, 2021.

\bibitem{xie2019zeno}
Cong Xie, Sanmi Koyejo, and Indranil Gupta.
\newblock Zeno: Distributed stochastic gradient descent with suspicion-based
  fault-tolerance.
\newblock In {\em International Conference on Machine Learning}, pages
  6893--6901. PMLR, 2019.

\bibitem{xie2020zeno++}
Cong Xie, Sanmi Koyejo, and Indranil Gupta.
\newblock Zeno++: Robust fully asynchronous sgd.
\newblock In {\em International Conference on Machine Learning}, pages
  10495--10503. PMLR, 2020.

\bibitem{yang2019federated}
Qiang Yang, Yang Liu, Tianjian Chen, and Yongxin Tong.
\newblock Federated machine learning: Concept and applications.
\newblock {\em ACM Transactions on Intelligent Systems and Technology (TIST)},
  10(2):1--19, 2019.

\bibitem{yeganeh2020inverse}
Yousef Yeganeh, Azade Farshad, Nassir Navab, and Shadi Albarqouni.
\newblock Inverse distance aggregation for federated learning with non-iid
  data.
\newblock In {\em Domain Adaptation and Representation Transfer, and
  Distributed and Collaborative Learning}, pages 150--159. Springer, 2020.

\bibitem{yin2018byzantine}
Dong Yin, Yudong Chen, Ramchandran Kannan, and Peter Bartlett.
\newblock Byzantine-robust distributed learning: Towards optimal statistical
  rates.
\newblock In {\em International Conference on Machine Learning}, pages
  5650--5659. PMLR, 2018.

\bibitem{YoonSHY21}
Tehrim Yoon, Sumin Shin, Sung~Ju Hwang, and Eunho Yang.
\newblock Fedmix: Approximation of mixup under mean augmented federated
  learning.
\newblock In {\em ICLR}, 2021.

\bibitem{zhang2021parameterized}
Jie Zhang, Song Guo, Xiaosong Ma, Haozhao Wang, Wenchao Xu, and Feijie Wu.
\newblock Parameterized knowledge transfer for personalized federated learning.
\newblock {\em Advances in Neural Information Processing Systems},
  34:10092--10104, 2021.

\bibitem{zhao2018federated}
Yue Zhao, Meng Li, Liangzhen Lai, Naveen Suda, Damon Civin, and Vikas Chandra.
\newblock Federated learning with non-iid data.
\newblock {\em arXiv preprint arXiv:1806.00582}, 2018.

\bibitem{zhou2017convergence}
Fan Zhou and Guojing Cong.
\newblock On the convergence properties of a $ k $-step averaging stochastic
  gradient descent algorithm for nonconvex optimization.
\newblock {\em arXiv preprint arXiv:1708.01012}, 2017.

\bibitem{zhou2021efficient}
Xiao Zhou, Weizhong Zhang, Zonghao Chen, Shizhe Diao, and Tong Zhang.
\newblock Efficient neural network training via forward and backward
  propagation sparsification.
\newblock {\em Advances in Neural Information Processing Systems},
  34:15216--15229, 2021.

\bibitem{zinkevich2010parallelized}
Martin Zinkevich, Markus Weimer, Lihong Li, and Alex~J Smola.
\newblock Parallelized stochastic gradient descent.
\newblock In {\em NIPS}, 2010.

\end{thebibliography}
}

\appendix

\clearpage
\onecolumn
\section*{\LARGE Supplementary Material}
\appendix
\begin{itemize}
    \item \autoref{suppl-sec:related}: additional related work (cf. Sec. 2 of the main paper).
    \item \autoref{suppl-sec-prf}: proof and additional analysis (cf. Sec. 4.2 of the main paper).
    \item \autoref{suppl-sec:exp_s}: additional details of experimental setups (cf. Sec. 5.1 of the main paper).
    \item \autoref{suppl-sec:exp_r}: additional experimental results and analysis (cf. Sec. 5.2 and 5.3 of the main paper).
    \item \autoref{suppl-sec:dis}: additional discussions (cf. 5.3 of the main paper).
\end{itemize}

\section{Additional Related Work}
\label{suppl-sec:related}
\subsection{Training in Subspace}
Several prior studies~\cite{DBLP:conf/iclr/LiFLY18,gur2018gradient,vinyals2012krylov} uncover the low-dimensionality essence in training neural networks, laying the foundation for the research on training in subspace. \cite{DBLP:conf/iclr/LiFLY18} first proposes to train networks in a smaller, randomly oriented subspace and demonstrate that the required dimension is much lower than the original dimension of parameters to obtain a relatively good performance.
Afterward, \cite{gressmann2020improving} proposes re-drawing the random subspace during training to improve the performance. Recently, \cite{li2022low} improves the random-oriented subspace by analyzing the optimization trajectory, and verifies that a carefully-extracted 40-dimensional space is enough to achieve comparable performance to regular training. The following study~\cite{li2022subspace} applies subspace training in adversarial training problems to prevent overfitting. 
In our work, we take advantage of the efficiency and generalization of subspace training to optimize server-side aggregation.
We leverage prior knowledge on aggregation for FL to construct the subspace as the convex hull spanned by client models.

\subsection{Federated Learning with Non-i.i.d. Data Distribution}
In this section, we supplement the other line of solutions discussed in the main paper for heterogeneous FL, i.e., \textbf{modifying local training and inference.}
Multiple branches of solutions are proposed to solve non-i.i.d data distribution through modifying local training and inference process. 
Several solutions propose to mitigate client drift through \textit{regularing local training}. FedPROX~\cite{li2020federated} and FedDYN~\cite{Acar2021Dyn} propose to regularize the drift of local model with global model. MOON~\cite{li2021model} introduces a contrastive loss and  SCAFFOLD~\cite{karimireddy2019scaffold} introduces control variates to correct local gradients.
\textit{Data sharing or augmentation} based solutions~\cite{shin2020xor, oh2020mix2fld,YoonSHY21,zhao2018federated} approach the problem from the data perspective and add to some shared/augmented data in local training to alleviate data heterogeneity.
\textit{Personalized FL}~\cite{kulkarni2020survey,t2020personalized,hanzely2020lower,li2021ditto} is also a branch of solutions that modify the local inference process. Instead of training a global model, these approaches seek to find the best local model, and the evaluation is performed locally.
Recently, a work~\cite{chen2021bridging} proposes to bridge generic FL and personalized FL to improve performance.
\subsection{Comparison with close FL work}

In this section, we detail the difference between SmartFL with existing FL works, which update coefficients for clients in aggregation at every communication round. For non-IID data distribution, FedPNS~\cite{wu2022node} and IDA~\cite{yeganeh2020inverse} reweight or select clients based on gradient diversity (since they focus on accelerating convergence instead of improving performance, we do not involve them in the comparison); ABAVG~\cite{xiao2021novel} updates coefficients for the clients based on accuracy on proxy data.
For relevant client selection, S-FedAVG~\cite{nagalapatti2021game} selects the relevant clients based on Shapley value calculated on proxy data. For personalized FL, KT-pFL~\cite{zhang2021parameterized} maintains a knowledge coefficient matrix for personalized knowledge transfer, which is not applicable in generic FL for an optimal global model.
For attack-robust aggregation, FLTrust~\cite{fltrust} and Sageflow~\cite{sageflow} assign weights for clients based on differences with the on-server model trained on proxy data and loss on proxy data, respectively. 
Overall, since existing reweighting-based FL methods heuristically rely on specifically-designed rules or special FL paradigms, they can only handle the targeted challenge, and the performance may not be optimal.
Differently, SmartFL learns the best coefficients at every communication round for aggregating the global model, which is essentially the optimization of the global model in a reduced subspace, which jointly handles potential problems caused by both challenges in real-world FL.
\clearpage
{

\section{Proof and additional analysis}
\label{suppl-sec-prf}
\subsection{Proof of Property 1}  
We prove the property 1 with the following  definitions and assumptions, which are widely adopted in the existing related studies~\cite{xie2019zeno, xie2020zeno++,sageflow}. 
\begin{definition}[L-smoothness]
We say a differentiable $f(\vw)$ L-smooth if there exists $L>0$ such that $$f(\vv) -f(\vw) \leq \langle \nabla f(\vw), \vv-\vw\rangle + \frac{L}{2}\|\vv-\vw\|^2, \forall \vw, \vv.$$
\end{definition}
\begin{definition}[Polyak-Łojasiewicz (PL) Inequality\cite{polyak1964gradient}] A function $f(\vw)$ satisfies the Polyak-Łojasiewicz (PL) inequality if 
there exists a constant $\mu>0$, such that 
$$f(\vw) - f(\vw^*) \leq \frac{1}{2\mu} \|\nabla f(\vw)\|^2, \forall \vw,$$
where $\vw^*$ is the minimum of $f(\vw)$.
\end{definition}
\begin{assumption}\label{assumption-honest}
We assume in each iteration $t$, there exists at least one honest client $i_t$ among the $M$ clients, who return the local models, in a sense that 
$$\langle \nabla \mathcal{L}_s(\vw^t, \mathcal{D}_s), \Delta_{i_t}^t\rangle + \gamma \|\nabla \mathcal{L}_s(\vw^t, \mathcal{D}_s)\|^2 \leq \epsilon,$$
where $\gamma>0$ and $\epsilon>0$ are  two  constants.
\end{assumption}
\begin{remark}
    Assumption \ref{assumption-honest} is practical and it is adopted in attack-robust studies \cite{xie2019zeno, xie2020zeno++}. It means that $\mathcal{L}_s(\vw^t, \mathcal{D}_s)$ can be reduced a little by involving $\Delta_{i_t}^t$ into $\vw^t$. If it is not satisfied in some extreme round, we can skip it and wait for the next communication round. 
\end{remark}

\begin{assumption} \label{ass:gap}
Given the client models $\vw_1, \ldots, \vw_M$, we assume $|\mathcal{L}(\vp,\mathcal{D})-\mathcal{L}_s(\vp, \mathcal{D}_s)| < \delta/2$ holds for a small constant $\delta>0$.
\end{assumption}
\begin{remark}
Note that $\mathcal{L}(\vp,\mathcal{D}) = \frac{1}{\left|\mathcal{D}\right|} \sum_{\xi \in \mathcal{D}} \ell(\boldsymbol{w}, \xi)$, and $\mathcal{L}_s\left(\boldsymbol{p}, \mathcal{D}_s\right)=\frac{1}{\mid \mathcal{D}_s \mid}  \sum_{\xi \in \mathcal{D}_s} \ell(\boldsymbol{w}, \xi)$, where $\vw$ is the weighted average of the given client models with coeffecient $\vp$. As our $\vp$ has a low dimension, $\mathcal{L}(\vp,\mathcal{D})$ can be approximated by $\mathcal{L}_s(\vp, \mathcal{D}_s)$ with a small subset $\mathcal{D}_s$ when these two datasets have similar distributions. We would like to point out that we need this assumption for the convenience of proof, however, in practice, we find that our method works well even if $\mathcal{D}_s$ is sampled from a different distribution. 

\end{remark}

Then, we would like to rephrase Property 1 into a more formal form below:
\begin{property}\label{property:convergence}
Besides Assumptions \ref{assumption-honest} and \ref{ass:gap}, we assume the losses $\mathcal{L}(\vw,\mathcal{D})$ and $\mathcal{L}_s(\vw,\mathcal{D}_s)$ are $L$-smooth and satisfy the PL inequality (potentially non-convex ). For the true and stochastic gradients, we assume that $\|\nabla \mathcal{L}(\vw,\mathcal{D})\|^2 \leq V_1, \|\nabla \mathcal{L}_s(\vw,\mathcal{D}_s)\|^2 \leq V_1$ and $\|\nabla \mathcal{L}_s(\vw,\mathcal{D}_s)-\nabla \mathcal{L}(\vw,\mathcal{D})\|^2 \leq V_3$. Further, we assume during training process, $\|\nabla \mathcal{L}_s(\vw^t,\mathcal{D}_s)\|^2$ is always low bounded, i.e., $\|\nabla \mathcal{L}_s(\vw^t,\mathcal{D}_s)\|^2>V_2>0$. Then, for our SmartFL, we have
\begin{align*}
    \mathbb{E}\left[\mathcal{L}(\vw^{T}, \mathcal{D}) - \mathcal{L}(\vw^*, \mathcal{D})\right] \leq &(1-2\mu\gamma \frac{V_2}{V_1} )^T  \mathbb{E}\left[\mathcal{L}(\vw^{0}, \mathcal{D}) - \mathcal{L}(\vw^*, \mathcal{D}) \right] + \frac{V_1}{2\mu \gamma V_2}\left[\eta\left(\frac{1}{2}V_3 + \frac{L+1}{2} V_1 \right)+  \epsilon + \delta \right],
\end{align*}
where $\eta <\min(1,1/L)$ is a small constant.
\end{property}
}
\begin{remark}
\textup{(1)} In Property \ref{property:convergence},following \cite{xie2020zeno++}, we assume $\|\nabla \mathcal{L}_s(\vw^t,\mathcal{D}_s)\|^2>V_2>0$ always holds during the training process. In practice, if we have a zero gradient, we can randomly discard/add some sample into the current minibatch to make it nonzero. \textup{(2)} Due to the PL inequlaity, we have $\mathcal{L}(\vw^T, \vw^*)\leq V_1/2\mu$. We would like to point out that as $\mu$ is always a small number, this bound is very loose and impractical. \textup{(3)} In our bound above, $\eta, \epsilon$ and $\delta$ are all small numbers, which indicates that the error can converge to a small value. 
\end{remark}
\begin{proof} of Property \ref{property:convergence}:\\
Denote the honest client in iteration $t$ to be $i_t$ and from Assumption \ref{assumption-honest}, we have
$$\langle \nabla \mathcal{L}_s(\vw^t,\mathcal{D}_s), \Delta_{i_t}^t\rangle \leq -\gamma \|\nabla \mathcal{L}_s(\vw^t, \mathcal{D}_s)\|^2  +\epsilon.$$
Thus, we can have
\begin{align}
    &\langle \nabla \mathcal{L}(\vw^t, \mathcal{D}), \Delta_{i_t}^t\rangle \nonumber \\
    &\leq \langle \nabla \mathcal{L}(\vw^t, \mathcal{D}) - \nabla \mathcal{L}_s(\vw^t, \mathcal{D}_s),  \Delta_{i_t}^t  \rangle -\gamma \|\nabla \mathcal{L}_s(\vw^t, \mathcal{D}_s)\|^2  + \epsilon\nonumber \\
    &\leq \langle \nabla \mathcal{L}(\vw^t, \mathcal{D}) - \nabla \mathcal{L}_s(\vw^t, \mathcal{D}_s),  \Delta_{i_t}^t  \rangle - \gamma \frac{V_2}{V_1} \|\nabla \mathcal{L}(\vw^t, \mathcal{D})\|^2  + \epsilon\nonumber \\
    & \leq \frac{\eta_{i_t}}{2}\|\nabla \mathcal{L}(\vw^t, \mathcal{D}) - \nabla \mathcal{L}_s(\vw^t, \mathcal{D}_s)\|^2 + \frac{\eta_{i_t}}{2}\|\nabla \mathcal{L}_s(\vw^t, \mathcal{D}_s)\|^2 - \gamma \frac{V_2}{V_1}\|\nabla \mathcal{L}(\vw^t, \mathcal{D})\|^2 +  \epsilon \nonumber \\
    & \leq -  \gamma \frac{V_2}{V_1}\|\nabla \mathcal{L}(\vw^t, \mathcal{D}) \|^2 + \frac{\eta_{i_t}}{2} V_1 + \frac{\eta_{i_t}}{2}V_3+ \epsilon. \nonumber
\end{align}
where $\eta_{i_t} = \|\Delta_{i_t}^t\|/\|\nabla \mathcal{L}_s(\vw^t, \mathcal{D}_s)\|$, which can be controlled by tuning the learning rate and length of local training. Therefore, we can assume $\eta_{i_t}\leq \eta <\min(1,1/L)$.

Notice that in our server-side aggregation, we search the model fusion in the convex hull spanned by the received client models, which contains these models. Therefore, we have 
\begin{align}
    &\mathcal{L}_s(\vw^{t+1}, \mathcal{D}_s)\leq \mathcal{L}_s(\vw^{t}_{i_t}, \mathcal{D}_s).\nonumber
    \end{align}
From Assumption \ref{ass:gap}, we can get
    \begin{align}
        &\mathcal{L}(\vw^{t+1}, \mathcal{D}) \leq \mathcal{L}_s(\vw^{t+1}, \mathcal{D}_s) + \delta/2 \leq \mathcal{L}_s(\vw^{t}_{i_t}, \mathcal{D}_s) +\delta/2 \leq \mathcal{L}(\vw^{t}_{i_t}, \mathcal{D}) + \delta.\nonumber
\end{align}
According to the smoothness of $\mathcal{L}(\vw, \mathcal{D})$, we can get \begin{align*}
    &\mathbb{E}\left[\mathcal{L}(\vw^{t+1}, \mathcal{D}) - \mathcal{L}(\vw^*, \mathcal{D})\right] \nonumber \\
    &\leq \mathbb{E}\left[\mathcal{L}(\vw^{t}_{i_t}, \mathcal{D}) - \mathcal{L}(\vw^*, \mathcal{D})\right] + \delta\nonumber\\
    &\leq  \mathbb{E}\left[\mathcal{L}(\vw^{t}, \mathcal{D}) - \mathcal{L}(\vw^*, \mathcal{D}) + \langle \nabla \mathcal{L}(\vw^t, \mathcal{D}), \Delta_{i_t}^t\rangle + \frac{L}{2}\|\Delta_{i_t}^t\|^2\right] +  \delta \nonumber\\
    & \leq  \mathbb{E}\left[\mathcal{L}(\vw^{t}, \mathcal{D}) - \mathcal{L}(\vw^*, \mathcal{D})-  \gamma \frac{V_2}{V_1}\|\nabla \mathcal{L}(\vw^t, \mathcal{D}) \|^2 \right] + \frac{\eta}{2}V_3 + \frac{L+1}{2}\eta V_1 + \eta \epsilon + \delta\nonumber \\
    & \leq (1-2\mu\gamma \frac{V_2}{V_1} ) \mathbb{E}\left[\mathcal{L}_s(\vw^{t}, \mathcal{D}_s) - \mathcal{L}_s(\vw^*, \mathcal{D}_s) \right] + \eta\left(\frac{1}{2}V_3 + \frac{L+1}{2} V_1 \right)  +  \epsilon + \delta.\nonumber
\end{align*}
Hence, for the model after $T$ aggregations, we can have 
\begin{align*}
    \mathbb{E}\left[\mathcal{L}(\vw^{T}, \mathcal{D}) - \mathcal{L}(\vw^*, \mathcal{D})\right] \leq &(1-2\mu\gamma \frac{V_2}{V_1} )^T  \mathbb{E}\left[\mathcal{L}(\vw^{0}, \mathcal{D}) - \mathcal{L}(\vw^*, \mathcal{D}) \right] + \frac{V_1}{2\mu \gamma V_2}\left[\eta\left(\frac{1}{2}V_3 + \frac{L+1}{2} V_1 \right)+  \epsilon + \delta \right]
\end{align*}
By choosing an appropriate $\gamma$ satisfying $0<1-2\mu\gamma \frac{V_2}{V_1}<1$, the expected error can converge linearly. 
\end{proof}

\subsection{Proof of Property 2}
\begin{property}[\textbf{Generalization in Aggregation}]\label{thm:generalization-formal}
We consider a binary classification problem with some mild conditions in \cite{ben2010theory}.  Assume $\Lambda$ contains $|\Lambda|$ discrete choices. Denote the dataset $\mathcal{D}_s^{-1}$ generated by replacing one sample
in $\mathcal{D}_s$ with another arbitrary sample. We assume there exists $\kappa>0$, such that  $|\mathcal{L}_s(\vw, \mathcal{D}_s)-\mathcal{L}_s(\vw, \mathcal{D}_s^{-1})|\leq \kappa/|\mathcal{D}_s|$ for all $\vw$. Given the received client models $\vw_m^{t}$, $m \in \mathcal{M}^t$ in round $t$, with the probability at least $1-\delta$, the server-side aggregations $\vw_{Smart}$ of  SmartFL satisfies the generalization upper bound:
 {\small
 \begin{align}
    \mathbb{E}_{\mathcal{D}} \mathcal{L}(\vw_{Smart}, \mathcal{D}) \leq \mathcal{L}_s(\vw_{Smart}, \mathcal{D}_s) +\kappa \sqrt{\frac{\ln(2|\Lambda|/\delta)}{2|\mathcal{D}_s|}}+C, \label{eqn:generalization-bound}
 \end{align}
 where $C$ comes the domain discrepancy between $\mathcal{D}$ and $\mathcal{D}_s$, i.e.,
 \begin{align}
 C=\frac{1}{2}d_{\mathcal{H}^t\Delta\mathcal{H}^t}(\tilde{\mathcal{D}}, \tilde{\mathcal{D}}_s) + \lambda,
 \end{align}
 with $\tilde{\mathcal{D}}$ and $\tilde{\mathcal{D}}_s$ being the distribution of $\mathcal{D}$ and $\mathcal{D}_s$, $d_{\mathcal{H}\Delta\mathcal{H}}$ being the domain discrepancy between two distributions, $\lambda = \min_{\vp} \mathbb{E}_{\mathcal{D}}\mathcal{L}(\vp, \mathcal{D}) + \mathcal{L}(\vp, \mathcal{D}_s)$. $\mathcal{H}^t$ the subspace in round $t$.
 }
\end{property}

\begin{lemma}[Domain Adaption\cite{ben2010theory}] Considering the distributions $\mathcal{D}_{S}$ and $\mathcal{D}_{T}$, for every $h \in \mathcal{H}$ and any $\delta\in (0,1)$, with probability at least $1-\delta$, there exists:
\begin{align}
    L_{\mathcal{D}_{T}}(h) \leq L_{\mathcal{D}_{S}}(h) + \frac{1}{2}d_{\mathcal{H} \Delta \mathcal{H}}(\mathcal{D}_S, \mathcal{D}_T) + \lambda,
\end{align}
where $\lambda = L_{\mathcal{D}_S}(h^*)+L_{\mathcal{D}_T}(h^*)$. $h^*:=\arg\min_{h\in \mathcal{H}} L_{\mathcal{D}_S}(h)+ L_{\mathcal{D}_T}(h)$. $d_{\mathcal{H} \Delta \mathcal{H}}$ measures the domain discrepancy between two distributions.
\end{lemma}
\begin{proof}
According to  the bounded difference inequality (Corollary 2.21 of \cite{wainwright2019high}), we can obtain:
{\begin{align}
    \mathbb{E}_{\mathcal{D}_s} \mathcal{L}_s(\vw_{Smart}, \mathcal{D}_s) \leq \mathcal{L}_s(\vw_{Smart}, \mathcal{D}_s) +\kappa \sqrt{\frac{\ln(2|\Lambda|/\delta)}{2|\mathcal{D}_s|}}.
 \end{align}}
 From the lemma above, we know that 
 \begin{align}
         \mathbb{E}_{\mathcal{D}} \mathcal{L}(\vw_{Smart}, \mathcal{D}) \leq \mathbb{E}_{\mathcal{D}_s} \mathcal{L}_s(\vw_{Smart}, \mathcal{D}_s)+ \frac{1}{2}d_{\mathcal{H}^t\Delta\mathcal{H}^t}(\tilde{\mathcal{D}}, \tilde{\mathcal{D}}_s) + \lambda.
 \end{align}
 Combining the above two inequalities, we have 
  \begin{align}
         \mathbb{E}_{\mathcal{D}} \mathcal{L}(\vw_{Smart}, \mathcal{D}) \leq \mathcal{L}_s(\vw_{Smart}, \mathcal{D}_s) +\kappa \sqrt{\frac{\ln(2|\Lambda|/\delta)}{2|\mathcal{D}_s|}}+ \frac{1}{2}d_{\mathcal{H}^t\Delta\mathcal{H}^t}(\tilde{\mathcal{D}}, \tilde{\mathcal{D}}_s) + \lambda.
 \end{align}

 \end{proof}

\section{Detailed Experiments Setups}
\label{suppl-sec:exp_s}

\subsection{Dataset}
CIFAR-10/100~\cite{krizhevsky2009learning} contain 50K training and 10K testing images for 10/100 class. MNIST~\cite{deng2012mnist} includes 60K training and 10K testing samples of written digits. FMNIST~\cite{xiao2017/online} includes 60K training and 10K testing samples of Zalando's article images.
The 20newsgroups~\cite{lang1995newsweeder} text dataset comprises around 20K news documents belonging to 20 categories, and it is split into 18K documents for training and 2000 documents for testing.

\subsection{Test Setting}
We give the results over three times of experiments and report mean $\pm$ standard deviation.

\subsection{Attack}
As mentioned in the main paper, we consider three kinds of attacks, including Label Flip Attack~\cite{fung2018mitigating}, Omniscient Attack~\cite{blanchard2017machine}, and Fang Attack~\cite{fang2020local}, which involves the data poisoning attack and model poisoning attack for FL. 
Specifically, Label Flip Attack switches the label to be the next class of the ground truth, while Omniscient Attack negates the original benign gradients. For Fang Attack, we adopt Median Attack as a representative attack considering byzantine-robust aggregation.
\subsection{Baselines}
\begin{itemize}
    \item FedAVG~\cite{mcmahan2017communication}: The standard communication-efficient aggregation strategy for federated learning.
    \item FedPROX~\cite{li2020federated}: An advanced method for heterogeneous federated learning technique that regularizes the drift of local model with the global model.
    \item Scaffold~\cite{karimireddy2019scaffold}: An advanced method for heterogeneous federated learning technique that introduces control variates to current local gradients.
    \item FedDF~\cite{lin2020ensembleFedDF}: An advanced aggregation strategy for heterogeneous federated learning using knowledge distillation with \textbf{unlabelled proxy data}.
    \item FedBE~\cite{chen2020fedbe}: An advanced aggregation strategy for heterogeneous federated learning using bayesian ensemble-based knowledge distillation with \textbf{unlabelled proxy data}.
    \item ABAVG~\cite{xiao2021novel}: An advanced aggregation strategy for heterogeneous federated learning using validation accuracy to reweight the clients with \textbf{labelled proxy data}.
    \item Finetuning: An advanced aggregation strategy for heterogeneous federated learning using \textbf{labelled proxy data} to finetune the aggregated model in every communication round, mentioned in \cite{chen2020fedbe}.
    \item Median~\cite{yin2018byzantine}: A Byzantine-robust aggregation strategy that calculates dimension-wise median for client updates.
    \item Krum~\cite{blanchard2017machine}: A Byzantine-robust aggregation strategy that vector-wisely selects an update.
    \item Trimmed Mean~\cite{yin2018byzantine}: A Byzantine-robust aggregation strategy that dimension-wisely removes a certain portion of the largest and smallest updates and calculates the mean of remaining values.
    \item Sageflow~\cite{sageflow}: A state-of-the-art attack-resistant aggregation strategy that combines entropy-based filtering and loss-based reweighting with \textbf{labelled proxy data.}
    \item FLTrust~\cite{fltrust}: A state-of-the-art attack-resistant aggregation strategy that maintains a server model, trains the server model with \textbf{labelled proxy data}, and reweights the client updates with the server update.
\end{itemize}
\subsection{Detailed Hyperparameter Setting}
\textbf{Baseline.} Generally, we follow the settings of the original papers without otherwise mentioning them. 
For the local training of FedPROX, we always tune the parameter according to the suggestion of the original paper to obtain the best performance for various conditions. 
For baseline models involving on-sever optimization with unlabelled/labelled data, the learning rate $\eta_s$ is tuned from $[5e-5, 1e-2]$, and the epochs is tuned from $E_s = \{1, 5, 10, 20\}$. Same as ours, the batch size is 32, and Adam Optimizer is used for on-server optimization. 
For FedBE, the sampling number for models is set to 10, according to the original paper.

\textbf{SmartFL \& SmartFL-U.} The default setting is mentioned in Section 5.1. 
We enlarge the server training epoch $E_s$ to be 50 for the experiment with attacks since the server-side optimization requires more steps to converge under poisoning attacks.

\section{Additional experiments}
\label{suppl-sec:exp_r}
\subsection{Robustness against data heterogeneity}
In this section, we include additional experiments on robustness against data heterogeneity, including \textbf{convergence speed} and the \textbf{extension to the NLP task}.

\subsubsection{Convergence Speed}
\label{supp-conve}
\begin{table*}[h]
    \centering
        \caption{Comparison of \textbf{the number of communication rounds} to reach target accuracy. We evaluate different FL methods with ResNet-8 on CIFAR-10 with different degrees of data heterogeneity $\alpha$ 
 and participation rates $C=0.4$. $^*$Methods assume the availability of unlabelled proxy data. $^\dagger$Methods assume the availability of labelled proxy data.}
    \label{tab:overview_con}
    \resizebox{\textwidth}{!}{%
\begin{tabular}{llllll}
\toprule
 & $\alpha = 0.01$&$\alpha = 0.04$& $\alpha = 0.16$&$\alpha = 0.32$&$\alpha = 0.64$ \\
Method& $target=0.35$&$target=0.57$& $target=0.68$&$target=0.72$&$target=0.735$ \\
\midrule
FedAVG     &   196.3$\pm$35.9 & 136.7$\pm$37.5 &  150.3$\pm$18.1      & 165.0$\pm$21.0 &  151.7$\pm$37.0\\
FedPROX    &   
  101.0$\pm$5.0 & 133.7$\pm$36.0 &  157.7$\pm$15.0      &    150.0$\pm$20.7 &  \underline{111.0$\pm$26.5}\\
  Scaffold    &   
  137.7$\pm$10.6 &   125.0$\pm$11.4      & 128.0$\pm$38.5 & 135.0$\pm$18.7 &  113.0$\pm$25.6\\
\midrule
FedDF$^*$  &   
  160.0$\pm$15.1 & 127.0$\pm$25.5 &  168.7$\pm$38.1     & 164.3$\pm$10.0 &  162.3$\pm$39.6\\
FedBE$^*$  &     
 182.3$\pm$15.9 & 132.0$\pm$31.2 &  143.7$\pm$32.9    & 177.0$\pm$10.0 &  146.7$\pm$49.2\\
\textbf{SmartFL-U$^*$}  &   135.3$\pm$22.0 & 117.7$\pm$11.2 &  \underline{91.0$\pm$11.1}       &153.0$\pm$33.0 &  124.0$\pm$46.5\\
\midrule
ABAVG$^\dagger$   & 
176.0$\pm$39.8 & 165.0$\pm$25.4 &  115.7$\pm$26.1      & \underline{149.3$\pm$15.0} &  129.7$\pm$14.6\\
Finetuning$^\dagger$    &   \underline{72.3$\pm$6.0} & \underline{96.0$\pm$12.2} &  97.3$\pm$3.8       &197.0$\pm$25.1 &  177.3$\pm$24.6\\
\textbf{SmartFL}$^\dagger$   &   \textbf{34.7$\pm$6.1} & \textbf{48.3$\pm$2.1} &  \textbf{58.3$\pm$1.5}       &\textbf{121.3$\pm$22.1} &  \textbf{98.0$\pm$17.6}\\
\bottomrule
\end{tabular}}
\end{table*}
Highly non-i.i.d. distribution of data also severely influences the convergence speed of standard aggregation strategies. \autoref{tab:overview_con} shows the number of communication rounds for the different methods to reach the target accuracy with ResNet-8 on CIFAR-10. Advanced aggregation strategies for heterogenous FL also accelerate convergence compared with FedAVG. SmartFL always requires much fewer communication rounds to achieve target performance in all conditions, indicating the efficiency and effectiveness of optimizing the aggregation via subspace training.

\subsubsection{Extension to NLP task}
\label{supp-nlp}
\begin{table}
	\centering
	\caption{Comparison of maximum test accuracy achieved by different methods with Logistic Regression on  20newsgroup with $C=40\%$.}
	\begin{tabular}{l|ccc}
	\toprule
    Methods &  $\alpha = 0.01$  & $\alpha=0.04$ & $\alpha=0.16$ \\
    \midrule

    FedAVG & 30.64$\pm$3.2 & 38.58$\pm$2.3 & 59.76$\pm$1.9\\
    \midrule
    FedDF$^*$ & 36.10$\pm$2.6 & 38.87$\pm$3.1 & 59.90$\pm$1.5\\
    SmartFL-U$^*$ & \underline{39.53$\pm$1.9} & \underline{43.10$\pm$1.8} & \underline{60.32$\pm$1.0}\\
    \midrule
    Finetune$^\dagger$ & 37.10$\pm$3.5 & 37.22$\pm$2.5 & 59.93$\pm$1.1\\
    SmartFL$^\dagger$ & \textbf{44.51$\pm$1.1} & \textbf{47.33$\pm$1.3} & \textbf{60.77$\pm$0.7} \\
    \bottomrule
	\end{tabular}
	\label{tbl:nlp}
\end{table}
 To verify the effectiveness of our method beyond the computer vision domain, we also evaluate our method using logistic regression on 20newsgroup~\cite{lang1995newsweeder}, a popular NLP benchmark for news classification.  As shown in \autoref{tbl:nlp}, SmartFL and SmartFL-U outperform the full-space training counterpart and FedAVG by a large margin across different $\alpha$ with both labelled and unlabelled proxy data.

\subsection{Robustness against attacks}
\label{supp-exp-robust}





This section includes more results and comprehensive analysis under different scenarios for the  MNIST and CIFAR-10 datasets in the setting mentioned in Section 5.3.
 \autoref{fig:exp-attack-mnist-more} and \autoref{fig:attack-cifar10-more} show a comparison of various aggregation strategies on MNIST and CIFAR-10 with high and normal data heterogeneity under Label Flip and Omniscient Attack. We also study Fang attack, which consider the robust aggregation in \autoref{fig:attack-fang}. We have the following observation classified by the methods:

First, statistical filtering-based Byzantine-robust methods such as Krum, Trimmed Mean, and Median can successfully defend against attacks in most cases when the attack rate is small and non-i.i.d. degree is not high, which is in line with the prior studies~\cite{yin2018byzantine,blanchard2017machine} However, they are not applicable when the attack rate get higher than half. Also, their performance is largely degraded when the data distribution is highly non-IID.

Second, the full-space training counterpart (i.e., Finetuning) performs relatively well among the methods on MNIST when the attack is not high but worse on CIFAR-10. This is because, for the simpler dataset, even overfitting on proxy data can to some extent help robust aggregation, while it does not work for the harder dataset. The results verify our intuition that finetuning massive parameters on a small amount of data can not dilute the negative effect brought by malicious clients.

Third, the methods leveraging server proxy data get the most competitive performance among all the solutions, suggesting the potential to improve the robustness of the server-side aggregation against attacks with reasonable server knowledge.
\begin{itemize}
    \item ABAVG~\cite{xiao2021novel}, which uses the validation accuracy on proxy data to reweight the clients, performs relatively well in defending against Label Flip Attack but fails to defend against modeling poisoning attacks. This is because, with a Label Flip attack, the attacker models are trained to predict a wrong label, and therefore the weight can be adjusted to a small value according to their low validation performance. However, for the model poisoning attacks, the validation performance is not necessarily low enough.
    \item Sageflow~\cite{sageflow}, which combines entropy-based filtering and loss-based reweighting, can get competitive performance under both types of attacks when the attack rate is not high and the distribution is not highly non-IID. However, it still fails in other conditions, especially with the Omniscient attack in that when the distribution is highly non-IID, the entropy of benign and malicious clients is not well separated.
    \item  FLTrust~\cite{fltrust} is the most competitive baseline that maintains a server model with proxy data and reweights the client updates according to the similarity with server model updates. We can observe that such a strategy enables robustness against attacks in almost all scenarios, especially model poisoning attacks, in that it can successfully capture and exclude the updates in an inverse direction of the server model. However, we still observe the instability of such a method during training since the stochastic gradient of the server model can not stably ensure ``good'' aggregation in all communication rounds. This can be a severe problem and sometimes leads to failure, as shown in \autoref{fail-fltrust}.
\end{itemize}

Finally, different from the above solutions that heuristically leverage server proxy data, we aggregate a global model with optimized combination coefficients for client models with proxy data in every communication round and stably mitigate the negative effects brought by malicious clients. 

\begin{figure}[h]
    \begin{subfigure}{0.49\linewidth}
        \centering
        \begin{subfigure}{0.49\linewidth}
        \includegraphics[width=1.0\linewidth]{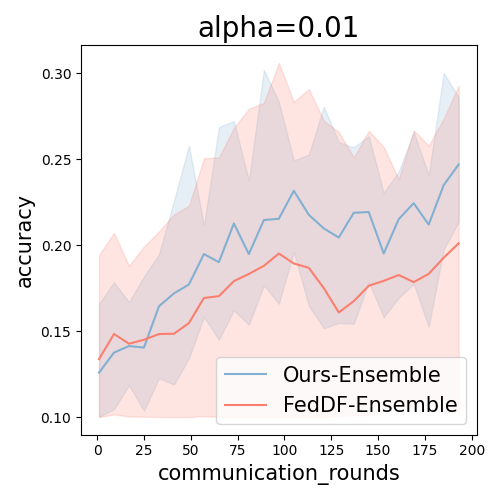}
        \end{subfigure}
        \begin{subfigure}{0.49\linewidth}
        \includegraphics[width=1.0\linewidth]{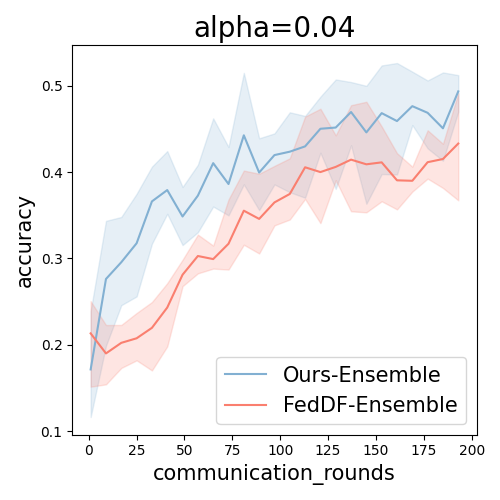}
        \end{subfigure}
        \subcaption{CIFAR-10}
    \end{subfigure}
    \begin{subfigure}{0.49\linewidth}
        \centering
        \begin{subfigure}{0.49\linewidth}
        \includegraphics[width=1.0\linewidth]{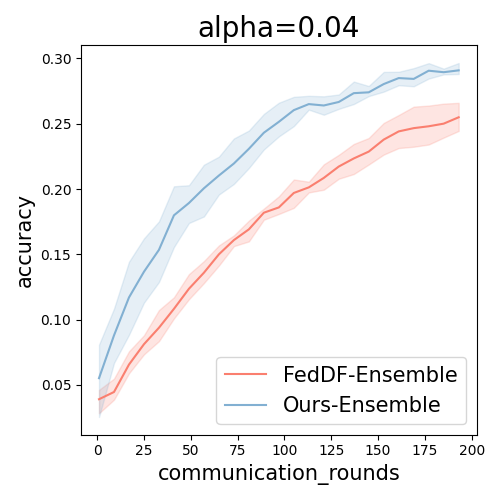}
        \end{subfigure}
        \begin{subfigure}{0.49\linewidth}
        \includegraphics[width=1.0\linewidth]{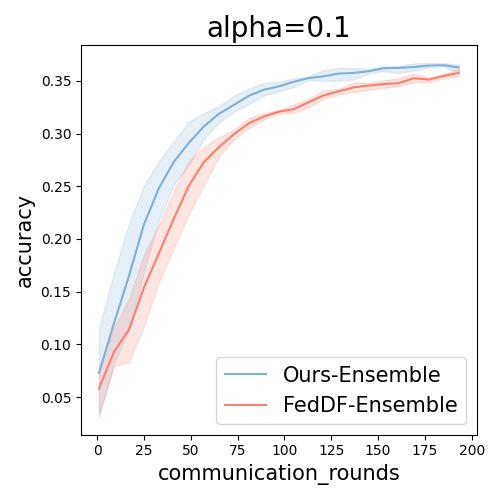}
        \end{subfigure}
        \subcaption{CIFAR-100}
    \end{subfigure}
    \caption{Studies on Heterogeneous Model Architectures (ResNet-8, MobileNet, and ShuffleNet). We compare our method with FedDF with unlabelled proxy data on CIFAR-10/100. We show the test accuracy of server ensemble model in every communication round.}
    \label{fig:hetero}
\end{figure}
\begin{figure*}[t]
    \centering
\includegraphics[width=0.95\linewidth]{Figures/aaaaaa.pdf}
    \begin{subfigure}{1.\linewidth}
    \begin{subfigure}{0.24\linewidth}
        \centering
        \includegraphics[width=1.0\linewidth]{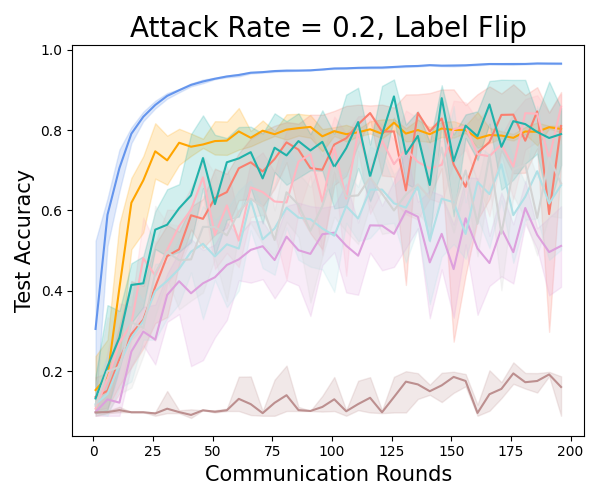}
    \end{subfigure}
        \begin{subfigure}{0.24\linewidth}
        \centering
        \includegraphics[width=1.0\linewidth]{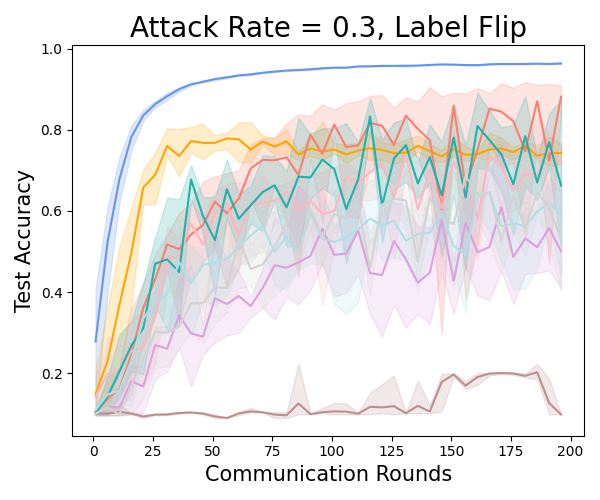}

    \end{subfigure}
    \begin{subfigure}{0.24\linewidth}
        \centering
        \includegraphics[width=1.0\linewidth]{Figures/mnist_LF_a0.01_r_0.4.png}
    \end{subfigure}
    \begin{subfigure}{0.24\linewidth}
        \centering
        \includegraphics[width=1.0\linewidth]{Figures/mnist_LF_a0.01_r_0.7.png} 
    \end{subfigure}
    \subcaption{MNIST, $\alpha=0.01$, Label Flip}

    \end{subfigure}
    \begin{subfigure}{1.\linewidth}
        \begin{subfigure}{0.24\linewidth}
        \centering
        \includegraphics[width=1.0\linewidth]{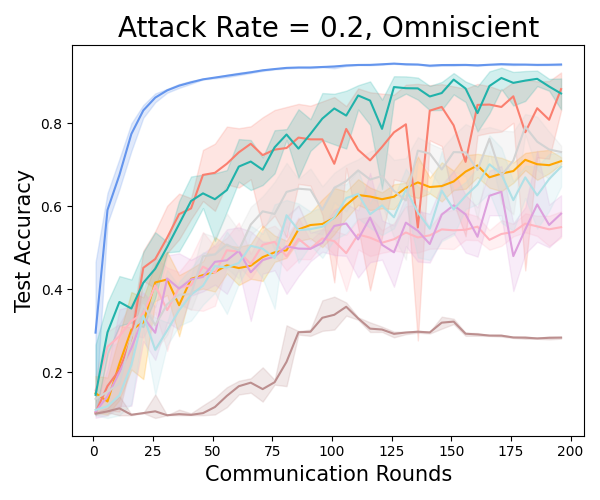}
        \end{subfigure}
        \begin{subfigure}{0.24\linewidth}
        \centering
        \includegraphics[width=1.0\linewidth]{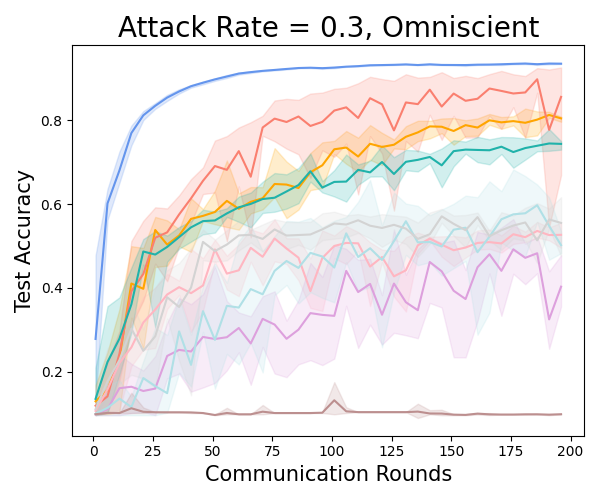}

    \end{subfigure}
    \begin{subfigure}{0.24\linewidth}
        \centering
        \includegraphics[width=1.0\linewidth]{Figures/mnist_Omi_a0.01_r_0.4.png}
        \end{subfigure}
    \begin{subfigure}{0.24\linewidth}
        \centering
        \includegraphics[width=1.0\linewidth]{Figures/mnist_Omi_a0.01_r_0.7.png} 
    \end{subfigure}
    \subcaption{MNIST, $\alpha=0.01$, Omniscient}
    \end{subfigure}        
        \begin{subfigure}{1.\linewidth}
    \begin{subfigure}{0.24\linewidth}
        \centering
        \includegraphics[width=1.0\linewidth]{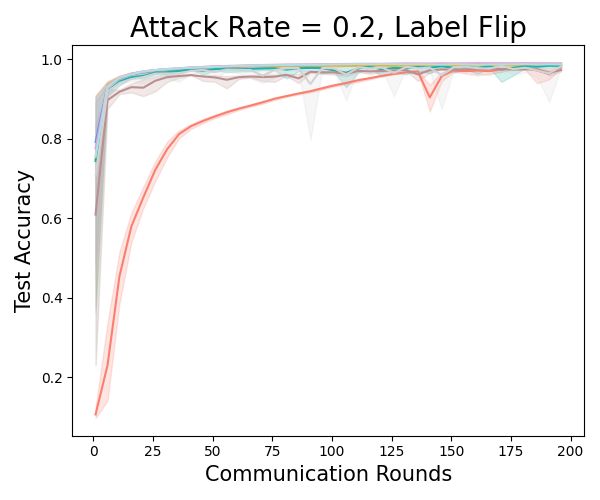}
    \end{subfigure}
        \begin{subfigure}{0.24\linewidth}
        \centering
        \includegraphics[width=1.0\linewidth]{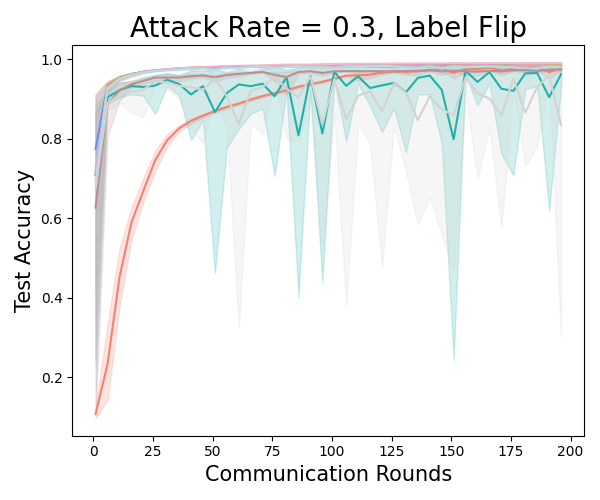}

    \end{subfigure}
    \begin{subfigure}{0.24\linewidth}
        \centering
        \includegraphics[width=1.0\linewidth]{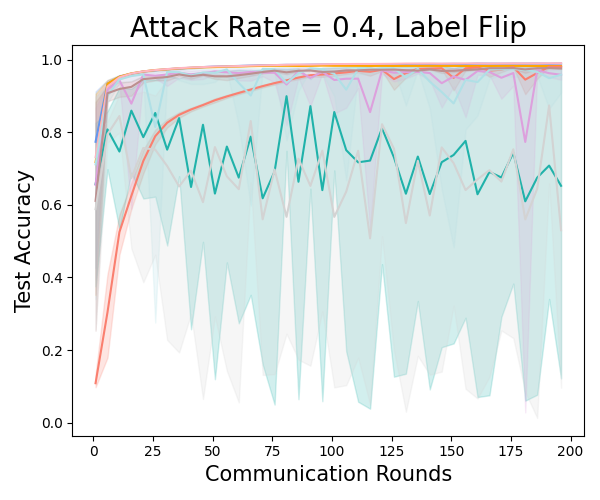}
    \end{subfigure}
    \begin{subfigure}{0.24\linewidth}
        \centering
        \includegraphics[width=1.0\linewidth]{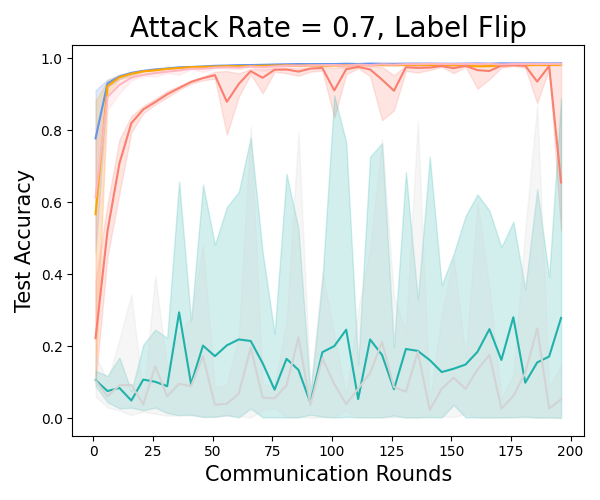} 
    \end{subfigure}
    \subcaption{MNIST, $\alpha=1$, Label Flip}
    \end{subfigure}
    
    \begin{subfigure}{1.\linewidth}
    \begin{subfigure}{0.24\linewidth}
        \centering
        \includegraphics[width=1.0\linewidth]{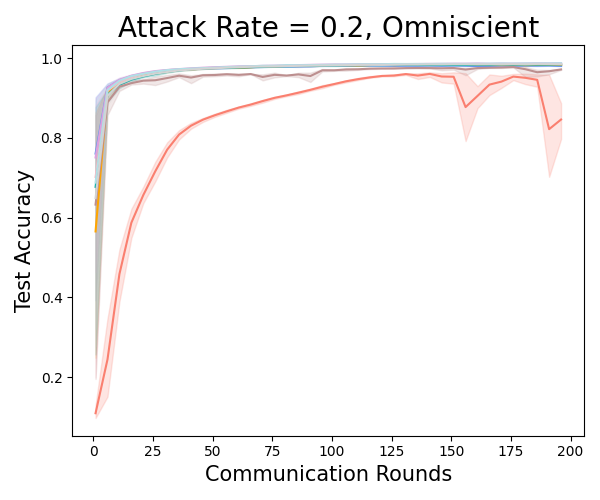}
    \end{subfigure}
        \begin{subfigure}{0.24\linewidth}
        \centering
        \includegraphics[width=1.0\linewidth]{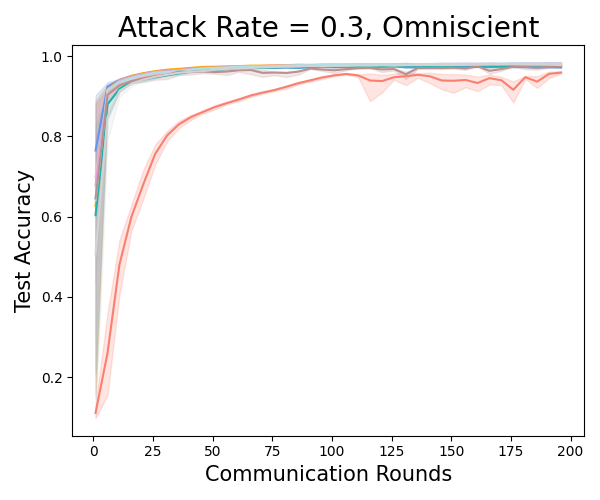}

    \end{subfigure}
    \begin{subfigure}{0.24\linewidth}
        \centering
        \includegraphics[width=1.0\linewidth]{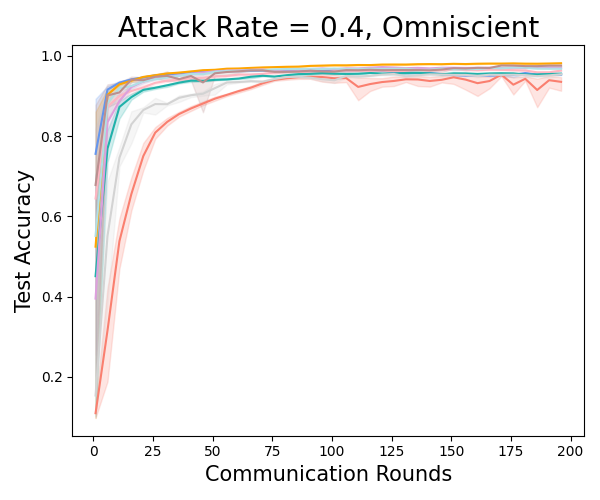}
    \end{subfigure}
    \begin{subfigure}{0.24\linewidth}
        \centering
        \includegraphics[width=1.0\linewidth]{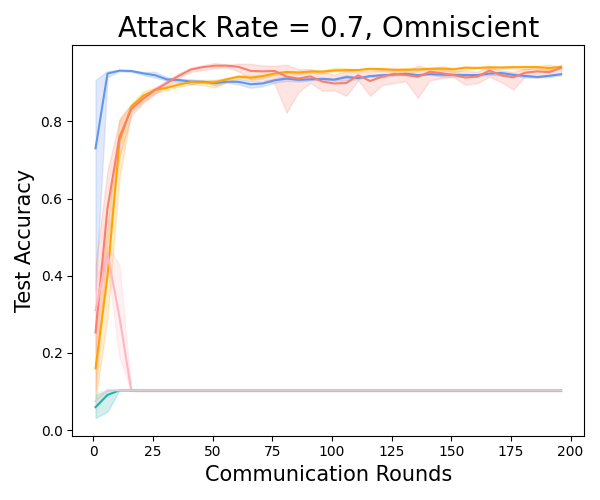} 
    \end{subfigure}
    \subcaption{MNIST, $\alpha=1$, Omniscient}
    \end{subfigure}
    \caption{Defence against Attacks on MNIST with the degree of data heterogeneity $\alpha = 0.01$ and $\alpha = 1$, under different types of attacks (Label Flip and Omniscient Attack) and different attack rates $AR \in \{0.2, 0.3,0.4,0.7\}$.}
    \label{fig:exp-attack-mnist-more}
\end{figure*}

\begin{figure*}[t]
    \centering
\includegraphics[width=0.95\linewidth]{Figures/aaaaaa.pdf}
    \begin{subfigure}{1.\linewidth}
    \begin{subfigure}{0.24\linewidth}
        \centering
        \includegraphics[width=1.0\linewidth]{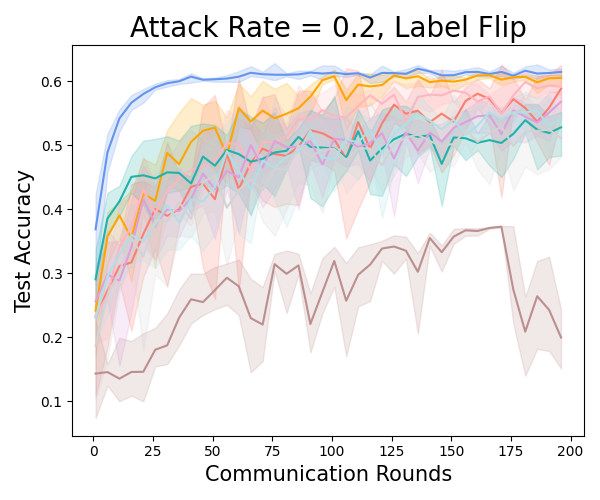}
    \end{subfigure}
        \begin{subfigure}{0.24\linewidth}
        \centering
        \includegraphics[width=1.0\linewidth]{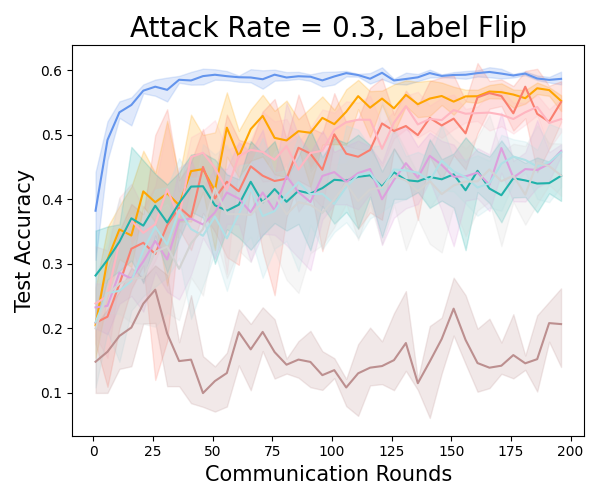}

    \end{subfigure}
    \begin{subfigure}{0.24\linewidth}
        \centering
        \includegraphics[width=1.0\linewidth]{Figures/cifar10_LF_a0.1_r_0.4.png}
    \end{subfigure}
    \begin{subfigure}{0.24\linewidth}
        \centering
        \includegraphics[width=1.0\linewidth]{Figures/cifar10_LF_a0.1_r_0.7.png} 
    \end{subfigure}
    \subcaption{CIFAR-10, $\alpha=0.1$, Label Flip}

    \end{subfigure}
    \begin{subfigure}{1.\linewidth}
        \begin{subfigure}{0.24\linewidth}
        \centering
        \includegraphics[width=1.0\linewidth]{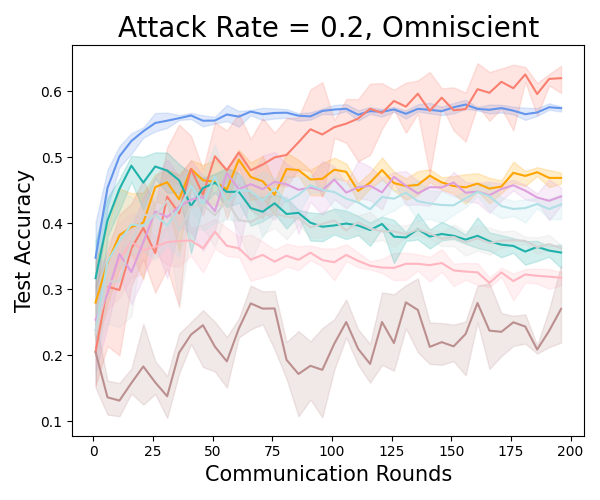}
        \end{subfigure}
        \begin{subfigure}{0.24\linewidth}
        \centering
        \includegraphics[width=1.0\linewidth]{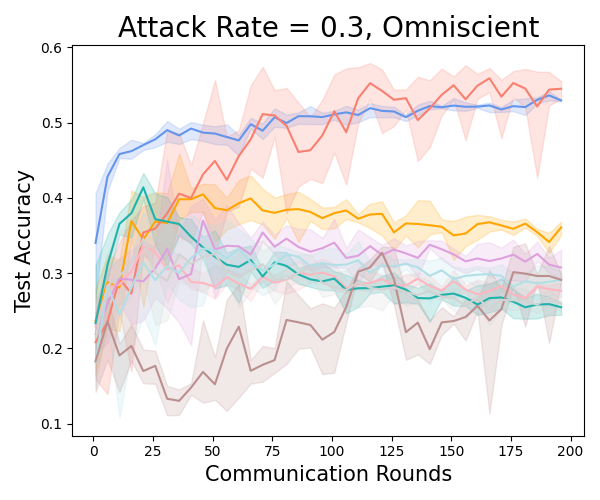}

    \end{subfigure}
    \begin{subfigure}{0.24\linewidth}
        \centering
        \includegraphics[width=1.0\linewidth]{Figures/cifar10_Omi_a0.1_r_0.4.png}
        \end{subfigure}
    \begin{subfigure}{0.24\linewidth}
        \centering
        \includegraphics[width=1.0\linewidth]{Figures/cifar10_Omi_a0.1_r_0.7.png} 
    \end{subfigure}
    \subcaption{CIFAR-10, $\alpha=0.1$, Omniscient}
    \end{subfigure}        
    \begin{subfigure}{1.\linewidth}
    \begin{subfigure}{0.24\linewidth}
        \centering
        \includegraphics[width=1.0\linewidth]{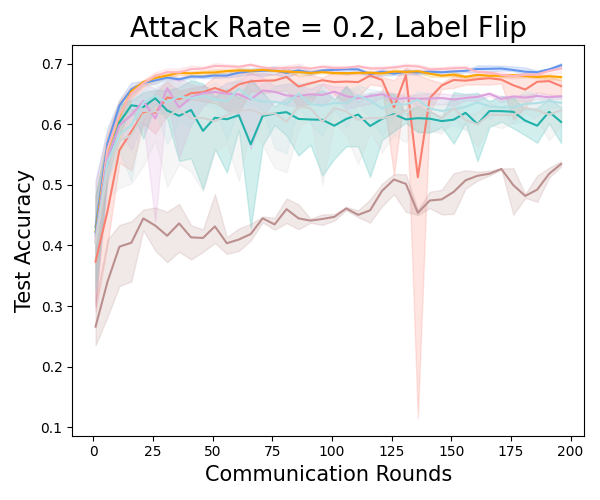}
    \end{subfigure}
        \begin{subfigure}{0.24\linewidth}
        \centering
        \includegraphics[width=1.0\linewidth]{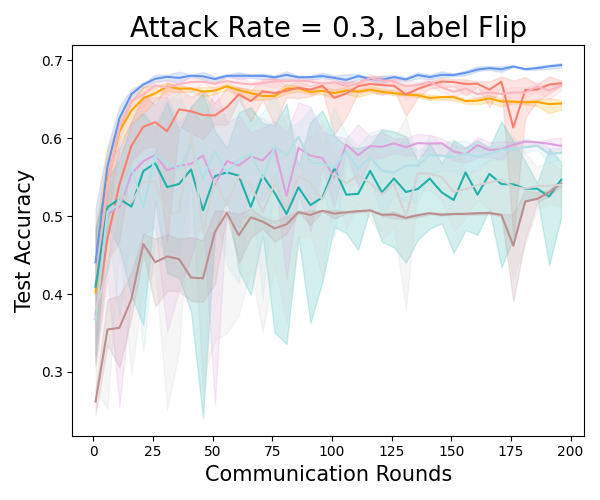}

    \end{subfigure}
    \begin{subfigure}{0.24\linewidth}
        \centering
        \includegraphics[width=1.0\linewidth]{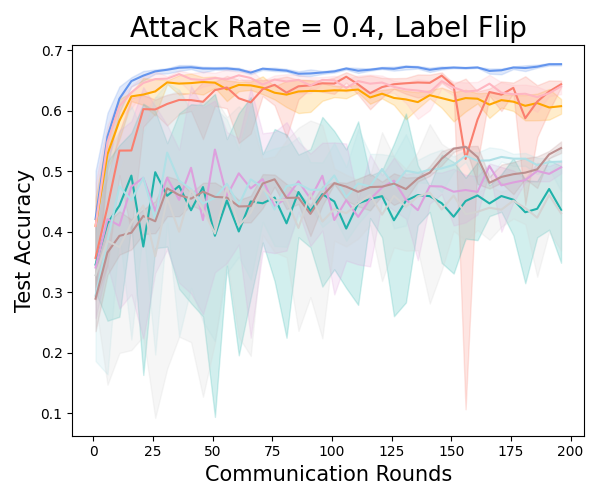}
    \end{subfigure}
    \begin{subfigure}{0.24\linewidth}
        \centering
        \includegraphics[width=1.0\linewidth]{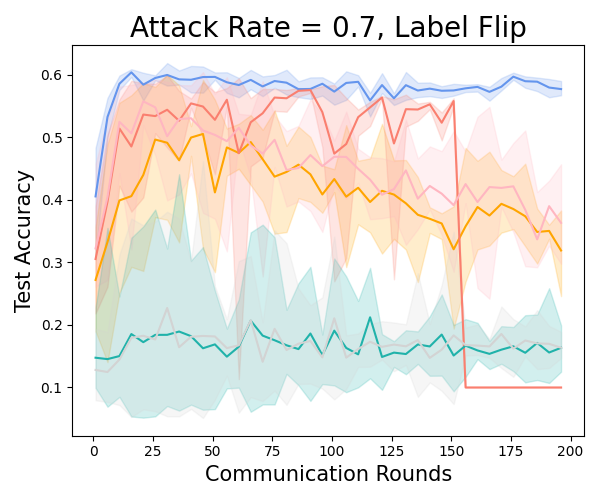} 
    \end{subfigure}
    
    \subcaption{CIFAR-10, $\alpha=1$, Label Flip}
    \label{fail-fltrust}
    \end{subfigure}
    \begin{subfigure}{1.\linewidth}
        \begin{subfigure}{0.24\linewidth}
        \centering
        \includegraphics[width=1.0\linewidth]{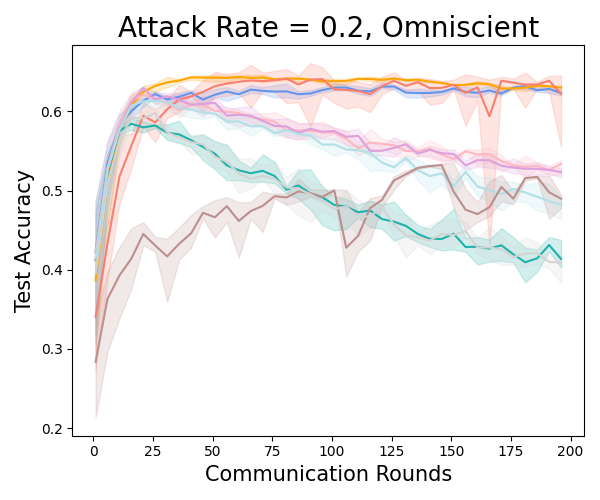}
        \end{subfigure}
        \begin{subfigure}{0.24\linewidth}
        \centering
        \includegraphics[width=1.0\linewidth]{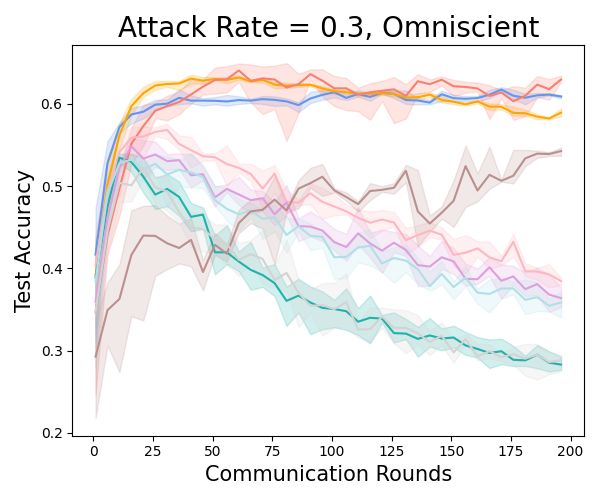}

    \end{subfigure}
    \begin{subfigure}{0.24\linewidth}
        \centering
        \includegraphics[width=1.0\linewidth]{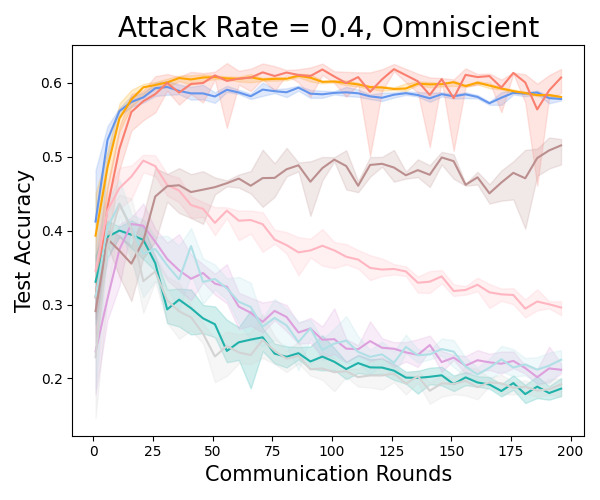}
        \end{subfigure}
    \begin{subfigure}{0.24\linewidth}
        \centering
        \includegraphics[width=1.0\linewidth]{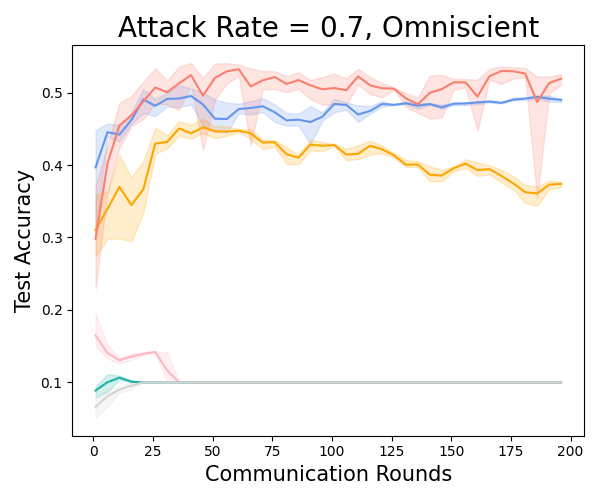} 
    \end{subfigure}
    \subcaption{CIFAR-10, $\alpha=1$, Omniscient}
    \end{subfigure}     

    

    \caption{Defence against Attacks on CIFAR-10 with the degree of data heterogeneity $\alpha = 0.1$ and $\alpha = 1$, under different types of attacks (Label Flip and Omniscient Attack) and different attack rates $AR \in \{0.2, 0.3,0.4,0.7\}$. }
    \label{fig:attack-cifar10-more}
\end{figure*}

\begin{figure*}[t]
    \centering
\includegraphics[width=0.95\linewidth]{Figures/aaaaaa.pdf}
    \begin{subfigure}{1.\linewidth}
    \begin{subfigure}{0.24\linewidth}
        \centering
        \includegraphics[width=1.0\linewidth]{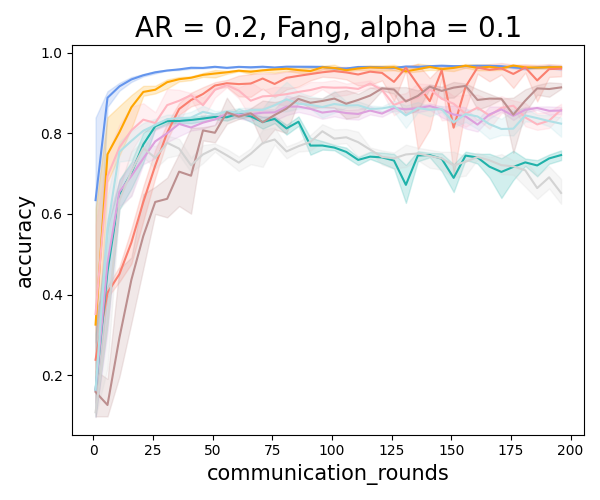}
    \end{subfigure}
        \begin{subfigure}{0.24\linewidth}
        \centering
        \includegraphics[width=1.0\linewidth]{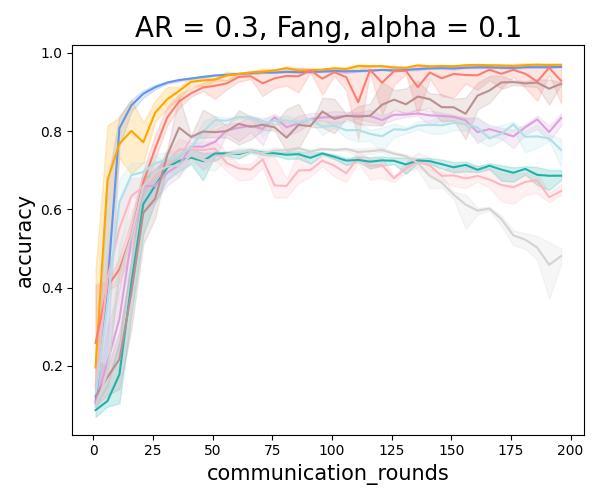}

    \end{subfigure}
    \begin{subfigure}{0.24\linewidth}
        \centering
        \includegraphics[width=1.0\linewidth]{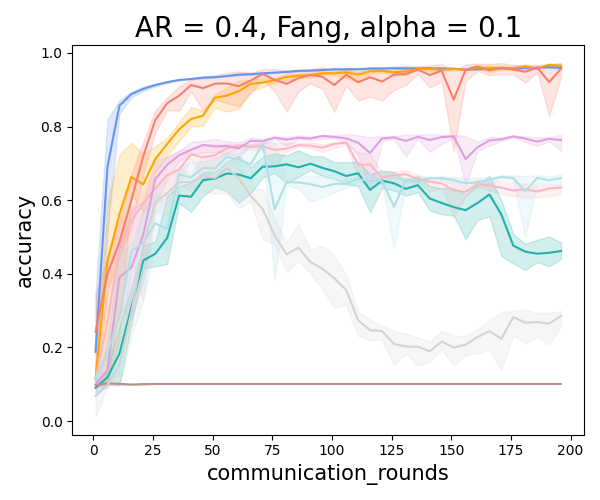}
    \end{subfigure}
    \begin{subfigure}{0.24\linewidth}
        \centering
        \includegraphics[width=1.0\linewidth]{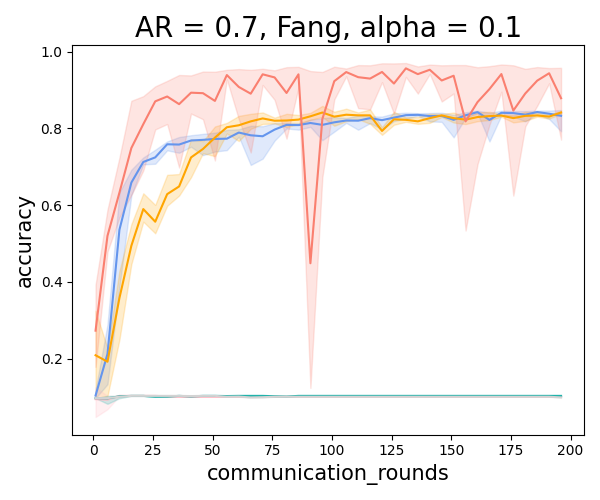} 
    \end{subfigure}
    \subcaption{MNIST, Median Attack}

    \end{subfigure}
    \begin{subfigure}{1.\linewidth}
        \begin{subfigure}{0.24\linewidth}
        \centering
        \includegraphics[width=1.0\linewidth]{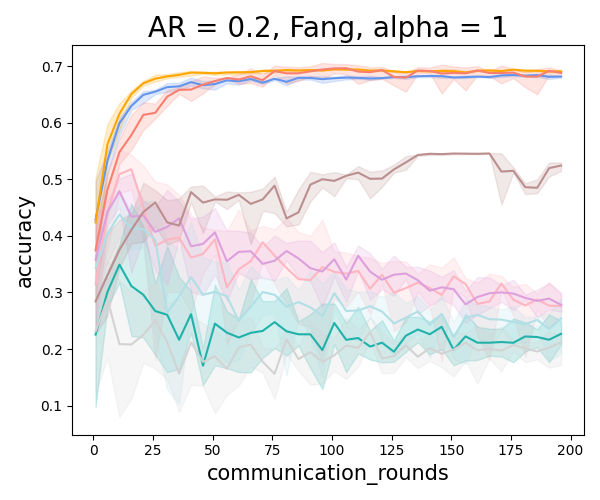}
        \end{subfigure}
        \begin{subfigure}{0.24\linewidth}
        \centering
        \includegraphics[width=1.0\linewidth]{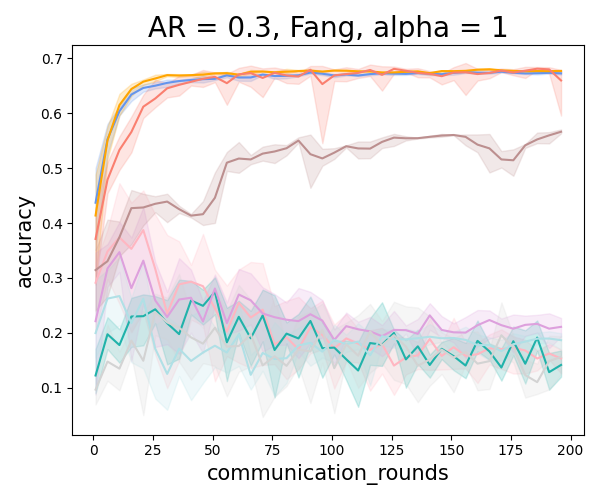}
    \end{subfigure}
    \begin{subfigure}{0.24\linewidth}
        \centering
        \includegraphics[width=1.0\linewidth]{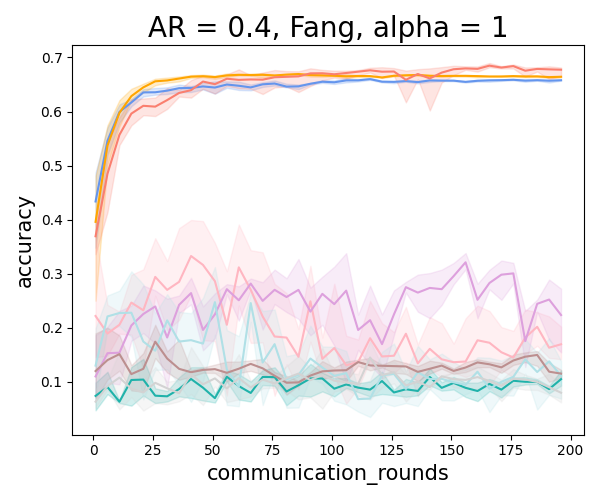}
        \end{subfigure}
    \begin{subfigure}{0.24\linewidth}
        \centering
        \includegraphics[width=1.0\linewidth]{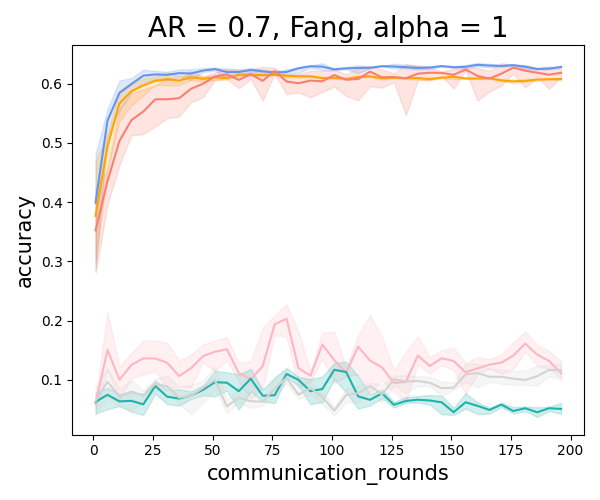} 
    \end{subfigure}
    \subcaption{CIFAR-10, Median Attack}
    \end{subfigure}     
    \caption{Defence against Fang Attack on MNIST/CIFAR-10 with the degree of data heterogeneity $\alpha = 0.1/1$, respectively, and different attack rates $AR \in \{0.2, 0.3,0.4,0.7\}$. }
    \label{fig:attack-fang}
\end{figure*}

\section{Discussions}
\label{suppl-sec:dis}
This section discusses the limitation and possible solutions.
Since we still optimize the combination weights for the local clients, one limitation of \text{SmartFL} is that the aggregated client model should be the same architecture and can not be directly applied on \textbf{heterogeneous model architectures}.
This can be alleviated by using multiple groups of model architectures. As illustrated in FedDF~\cite{lin2020ensembleFedDF}, knowledge distillation on unlabelled data using ensemble logits can allow information flow across models of different groups of architectures, and the server can use the ensemble of aggregated global models to make the final prediction.
Here we show that our solution for unlabelled data (SmartFL-U) shares the merits of regular knowledge distillation~\cite{lin2020ensembleFedDF} in allowing information flow across heterogeneous neural architectures~\cite{li2019fedmd} by using the ensemble logits of all clients to supervise the combination with groups.
\autoref{fig:hetero} visualizes the test accuracy in every communication round of ensemble performance of SmartFL and the state-of-the-art FedDF for heterogeneous model architectures (ResNet-8, MobileNet, and ShuffleNet) with 128 unlabelled data on CIFAR-10, and 512 unlabelled data on CIFAR-100. SmartFL consistently dominates FedDF, demonstrating the effectiveness of breaking the knowledge barrier of heterogeneous models by leveraging averaged logits to optimize the global models in the subspace. We leave the possible improvement through leveraging both ground truth labels and ensemble client knowledge as future work.



\end{document}